\documentclass[10pt]{article} 
\usepackage[preprint]{tmlr}

\usepackage{amsmath,amsfonts,bm}

\def\eqref#1{equation~\ref{#1}}

\def\1{\bm{1}}

\DeclareMathAlphabet{\mathsfit}{\encodingdefault}{\sfdefault}{m}{sl}
\SetMathAlphabet{\mathsfit}{bold}{\encodingdefault}{\sfdefault}{bx}{n}

\usepackage{amsthm}
\usepackage{hyperref}
\usepackage{url}
\usepackage{algorithm,xcolor}
\usepackage{algpseudocode}
\usepackage{graphicx}
\usepackage{subcaption}
\usepackage{caption}

\newtheorem{theorem}{Theorem}[section]
\newtheorem{proposition}[theorem]{Proposition}

\title{Gradient-Discrepancy Acquisition for Pool-Based Active Learning}

\author{\name Mohamadsadegh Khosravani \email mko41@uregina.ca \\
      \addr Department of Computer Science\\
      University of Regina
      \AND
      \name Sandra Zilles \email zilles@uregina.ca \\
      \addr University of Regina
      }

\def\month{MM}  
\def\year{YYYY} 
\def\openreview{\url{https://openreview.net/forum?id=XXXX}} 

\begin{document}

\maketitle

\begin{abstract}
The effectiveness of 
active learning hinges on the choice of the acquisition criterion by which a learning algorithm selects potentially informative data points whose label is subsequently queried.  
This paper proposes a novel gradient-based acquisition criterion, derived from a formal generalization bound proven by Luo et al.\ (2022). This criterion can be applied as an alternative to uncertainty measures in uncertainty sampling. We provide a theoretical justification of the proposed acquisition criterion, and demonstrate its effectiveness in an empirical evaluation. 
\end{abstract}
\section{Introduction}

In many supervised learning applications, labels are expensive to obtain while large unlabeled datasets are readily available. Pool-based active learning (AL) addresses this setting by iteratively selecting a small batch of unlabeled points to be labeled by an oracle, with the goal of maximizing test performance under a limited labeling budget~\citep{settles2009active}. A central design choice in AL is the acquisition function used to decide which points to label.

Two families of acquisition functions are widely used. \emph{Uncertainty-based} methods query points the current model finds ambiguous, using statistics of the predictive distribution such as entropy or margin~\citep{settles2009active,gal2017deep}. \emph{Diversity/coverage-based} methods aim to spread queries so as to sample from the unlabeled pool across a large variety of areas within either the data space or the feature space, for example via a core-set style selection 
\citep{sener2017active}. While effective in many settings, both uncertainty-based and diversity-based heuristics can be sensitive to dataset shift and to the choice of model and experimental setting~\citep{ovadia2019can,lowell2019practical}.

Another perspective uses gradients to guide acquisition, treating the gradient induced by a candidate point as a representation of the training signal or model update that point would produce. This perspective is largely orthogonal to the uncertainty/diversity taxonomy: some gradient-based methods are uncertainty-like, while others use gradients to promote diversity. Expected Gradient Length (EGL) scores an unlabeled point by the magnitude of the update it is expected to induce, using an expectation over possible labels~\citep{NIPS2007_a1519de5}. For deep models, gradient-based acquisition is often implemented through \emph{last-layer gradient embeddings}, which are efficient to compute and compare. BADGE computes pseudo-labeled last-layer gradients and then selects a batch using clustering (k-means++) in gradient space to encourage within-batch diversity~\citep{Ash2020}. Related ideas include gradient matching for subset selection~\citep{DBLP:journals/corr/abs-2103-00123}, and gradient-based influence measures for identifying informative examples~\citep{DBLP:journals/corr/abs-2107-07075}. Note that gradient matching is commonly studied as an \emph{offline} subset selection objective.

\textbf{Our approach.}
In this paper, we propose a \emph{gradient-discrepancy} acquisition heuristic inspired by a term that appears in a gradient-based generalization bound~\citep{luo2022generalization}. We do not claim that directly optimizing this bound term yields a provable reduction in the generalization gap for deep networks. Instead, we use the bound-motivated quantity as a practical scoring rule.

For each unlabeled point, we compute a pseudo-labeled last-layer gradient embedding and measure its discrepancy to a reference statistic computed from the labeled set (e.g., the mean labeled gradient embedding). Intuitively, we prioritize points whose gradients are most misaligned with what the current labeled set ``explains,'' treating larger discrepancy as a proxy for informativeness. This yields the acquisition rule \texttt{grad}, which selects the top-$b$ points by the discrepancy score.

\textbf{Contributions.}
Our main contributions are:
\begin{itemize}
  \item We define a gradient-discrepancy acquisition score inspired by the theoretical study of \cite{luo2022generalization}, and implement it efficiently using pseudo-labeled last-layer gradient embeddings.
  \item We propose a simple batch selection rule, \texttt{grad}, derived from the discrepancy score.
  \item We evaluate this strategy across multiple datasets, architectures, and acquisition batch sizes, and compare against common baselines including uncertainty sampling, BADGE, and core-set selection.
  Our evaluation shows that our gradient-discrepancy score typically outperforms entropy-based scoring in uncertainty sampling, and is overall comparable to other standard active learning methods, in terms of resulting predictive accuracy.
\end{itemize}

\paragraph{Method taxonomy.}
Table~\ref{tab:al_taxonomy} summarizes how our discrepancy criterion differs from common uncertainty-, diversity-, and gradient-embedding-based baselines.

\begin{table}[t]
\centering
\small
\setlength{\tabcolsep}{4pt}
\begin{tabular}{l p{3.6cm} p{2.9cm} p{3.8cm}}
\hline
\textbf{Method} & \textbf{Signal} & \textbf{Selection mechanism} & \textbf{Key intuition} \\
\hline
Entropy / Margin & predictive uncertainty & top-$b$ by uncertainty & query ambiguous points \\
Core-set / K-center~\citep{sener2017active} & feature-space distances & greedy coverage & cover the pool in representation space \\
EGL~\citep{NIPS2007_a1519de5} & expected gradient \emph{magnitude} & top-$b$ by score & maximize expected model change \\
BADGE~\citep{Ash2020} & pseudo-labeled gradient embeddings & k-means++ in gradient space & promote within-batch diversity in gradient space \\
\hline
\texttt{grad} (ours) & gradient \emph{discrepancy to labeled reference} & top-$b$ by discrepancy & query points most misaligned with labeled-set gradient statistics \\

\hline
\end{tabular}
\caption{Taxonomy of acquisition strategies. While BADGE and our method both use last-layer gradient embeddings, BADGE selects a diverse batch via clustering, whereas we score points by discrepancy to a labeled-set reference statistic.}
\label{tab:al_taxonomy}
\end{table}

\section{Preliminaries}
\label{Preliminaries}
\subsection{Pool-based Active Learning}

We consider pool-based active learning for supervised classification.
Let $\mathcal{D}=\{(x_i,y_i)\}_{i=1}^N$ denote the underlying dataset, where labels $y_i$ are initially unknown to the learner except for a small labeled subset.
At acquisition round $t$, the learner maintains a labeled set $\mathcal{D}_L^{(t)}$ and an unlabeled pool $\mathcal{D}_U^{(t)}$ with $\mathcal{D}_L^{(t)} \cup \mathcal{D}_U^{(t)}=\mathcal{D}$ and $\mathcal{D}_L^{(t)} \cap \mathcal{D}_U^{(t)}=\emptyset$.
Given a labeling budget $b$, an acquisition function selects a batch $\mathcal{B}^{(t)} \subset \mathcal{D}_U^{(t)}$ with $|\mathcal{B}^{(t)}|=b$ to be labeled by an oracle.
The labeled set is then updated as $\mathcal{D}_L^{(t+1)}=\mathcal{D}_L^{(t)}\cup \mathcal{B}^{(t)}$, and the unlabeled pool becomes $\mathcal{D}_U^{(t+1)}=\mathcal{D}_U^{(t)}\setminus \mathcal{B}^{(t)}$.
At each round, a model is trained using $\mathcal{D}_L^{(t)}$, and performance is evaluated on a held-out test set.

\subsection{Model, Loss, and Gradient}

Let $p_\theta(y\mid x)$ be a parametric predictive model with parameters $\theta \in \mathbb{R}^d$.
For a labeled example $(x,y)$, let $\ell(\theta;x,y)$ denote a differentiable training loss (e.g., cross-entropy).
For any finite set of labeled examples $S=\{(x_i,y_i)\}_{i=1}^{|S|}$, define the empirical objective
\begin{equation}
f(\theta; S) \;=\; \frac{1}{|S|}\sum_{(x_i,y_i)\in S}\ell(\theta;x_i,y_i),
\end{equation}
and its gradient
\begin{equation}
\nabla f(\theta; S) \;=\; \nabla_\theta f(\theta; S)
\;=\; \frac{1}{|S|}\sum_{(x_i,y_i)\in S}\nabla_\theta \ell(\theta;x_i,y_i).
\end{equation}

For an unlabeled point $x \in \mathcal{D}_U^{(t)}$, we use a pseudo-label $\hat{y}_\theta(x)$ induced by the current model, and write $\hat{S}=\{(x,\hat{y}_\theta(x))\}$ when treating $x$ as a singleton labeled set for scoring purposes.
Unless stated otherwise, all gradients are taken with respect to the full parameter vector $\theta$.

\subsection{\texorpdfstring{$\mathcal{DF}$}{DF} Discrepancy and Pseudo-Labeling}
\label{subsection: discrepancy}
\paragraph{Pseudo-labeling.}
For an unlabeled point $x \in \mathcal{D}_U^{(t)}$, we assign a deterministic pseudo-label using the current model:
\begin{equation}
\hat{y}_{\theta}(x) \;=\; \arg\max_{y} \; p_{\theta}(y\mid x).
\end{equation}
We denote the corresponding pseudo-labeled singleton set by
\begin{equation}
\hat{S}_{\theta}(x) \;=\; \{(x,\hat{y}_{\theta}(x))\}.
\end{equation}

\paragraph{Gradient-discrepancy function.}
Given two (pseudo-)labeled sets $S$ and $T$, we define the \emph{gradient discrepancy} function
\begin{equation}
\mathcal{DF}_{\theta}(S,T) \;=\; \nabla f(\theta; S) \;-\; \nabla f(\theta; T),
\end{equation}
where $\nabla f(\theta; S)$ denotes the gradient of the empirical objective on $S$ with respect to $\theta$.
When the dependence on $\theta$ is clear from context (e.g., $\theta=\theta_t$ at acquisition round $t$), we write $\mathcal{DF}(S,T)$ for brevity.

\subsection{A Generalization Bound Based on Gradients}

We briefly review a gradient-based generalization bound by \cite{luo2022generalization}, which will serve as
\emph{inspiration} for the discrepancy primitive used later in our acquisition rules.

Let $S=(z_1,\ldots,z_n)\sim \mathcal{D}^n$ be a training sample, and let $J=(j_1,\ldots,j_m)$ be a random
sequence of $m$ indices sampled uniformly from $[n]$. Denote by $S_J=(z_{j_1},\ldots,z_{j_m})$ the
corresponding subsequence, and let $S_I$ with $I=[n]\setminus J$ be the complementary subsequence. 
\cite{luo2022generalization} introduce Floor Gradient Descent (FGD) and Floor Stochastic Gradient Descent (FGSD),
gradient methods designed to facilitate generalization analysis.

In their setting, with probability at least $1-\delta$ over the draw of $S$ and the random index sequence $J$,
the iterated model $W_T$ after $T$ steps satisfies a bound of the form
\begin{equation}
R(W_T,\mathcal{D})
\;\le\;
\eta C_{\eta}\, R(W_T,S_I)
\;+\;
\frac{C_{\eta}\,\ln(1/\delta)+3}{n-m}
\;+\;
\frac{C_{\eta}\,\ln(dT)}{n-m}
\sum_{t=1}^{T}
\left(
\frac{\gamma_t^2}{\epsilon_t^2}\,
\left\|\mathcal{DF}(S,S_J,t)\right\|_2^2
\right),
\label{eq:luo_bound}
\end{equation}
where $R(\cdot,\cdot)$ denotes (empirical or population) risk, $d$ is the parameter dimension,
$\{\gamma_t\}_{t=1}^T$ is a step-size sequence, $\{\epsilon_t\}_{t=1}^T$ is a precision sequence specific to
FGD/FGSD, $\eta > 0 $ is a constant and $C_{\eta}$ is a constant depending on $\eta$ (see \citep{luo2022generalization} for details). In other words, the population risk after $T$ steps is bounded from above by a constant multiple of the empirical risk plus a function linear in $1/\delta$, plus a third term that grows with a cumulative \emph{gradient-discrepancy} term $\sum_{t=1}^{T}
\left(
\frac{\gamma_t^2}{\epsilon_t^2}\,
\left\|\mathcal{DF}(S,S_J,t)\right\|_2^2
\right)$, where
\begin{equation}
\mathcal{DF}(S,S_J,t)
\;:=\;
\nabla f(\theta_t; S)\;-\;\nabla f(\theta_t; S_J)
\;=\;
\mathcal{DF}_{\theta_t}(S,S_J),
\label{eq:df_in_bound}
\end{equation}
i.e., a difference between mean gradients of two sets evaluated at the current iterate $\theta_t$.

A key design idea of our method is to try to limit the growth of the gradient-discrepancy term. We emphasize that, in our work, the appearance of $\mathcal{DF}$ in \eqref{eq:luo_bound} is used as a
\emph{motivating structural primitive} for designing practical acquisition heuristics, rather than as a claim
that directly optimizing the bound term yields a guaranteed reduction in the generalization gap for deep
networks. 

\paragraph{A simplified form under contraction of $\mathcal{DF}$.}
A key term in \eqref{eq:luo_bound} is the cumulative gradient-discrepancy sum
$\sum_{t=1}^{T} \frac{\gamma_t^2}{\epsilon_t^2}\|\mathcal{DF}(S,S_J,t)\|_2^2$.
Motivated by the appearance of $\mathcal{DF}$ in \eqref{eq:luo_bound}, we consider a regime in which this
discrepancy becomes non-increasing across iterations. In particular, we can drastically simplify the gradient-discrepancy sum under the following assumption:

\paragraph{Assumption 1 (Monotone contraction of $\mathcal{DF}$).}
\label{ass:df_monotone}
There exists an iteration index $t_0 \in \{1,\ldots,T\}$ such that for all $t \ge t_0$,
\begin{equation}
\|\mathcal{DF}(S,S_J,t+1)\|_2 \;\le\; \|\mathcal{DF}(S,S_J,t)\|_2 .
\label{eq:df_monotone}
\end{equation}
Appendix~\ref{sec:appendix_contraction} provides sufficient conditions under which
Assumption 1 holds.

Under Assumption~$1$, we have $\|\mathcal{DF}(S,S_J,t)\|_2^2 \le \|\mathcal{DF}(S,S_J,t_0)\|_2^2$
for all $t \ge t_0$, and therefore
\begin{align}
\sum_{t=t_0}^{T} \frac{\gamma_t^2}{\epsilon_t^2}\,\|\mathcal{DF}(S,S_J,t)\|_2^2
&\le
\left(\max_{t \in \{t_0,\ldots,T\}} \frac{\gamma_t^2}{\epsilon_t^2}\right)
(T-t_0+1)\,\|\mathcal{DF}(S,S_J,t_0)\|_2^2 .
\label{eq:df_sum_bound}
\end{align}

Plugging \eqref{eq:df_sum_bound} into \eqref{eq:luo_bound} yields the simplified inequality
\begin{equation}
R(W_T,\mathcal{D})
\;\le\;
\eta C_{\eta}\, R(W_T,S_I)
\;+\;
\frac{C_{\eta}\,\ln(1/\delta)+3}{n-m}
\;+\;
\frac{C_{\eta}\,\ln(dT)}{n-m}\;
\kappa_{t_0}\;
\|\mathcal{DF}(S,S_J,t_0)\|_2^2,
\label{eq:luo_bound_simplified}
\end{equation}
where $\kappa_{t_0}:=(T-t_0+1)\max_{t\in\{t_0,\ldots,T\}}\frac{\gamma_t^2}{\epsilon_t^2}$.

\section{\texorpdfstring{$\mathcal{DF}$}{DF}-based Acquisition Strategies}

We now describe a pool-based acquisition strategy built from the $\mathcal{DF}$ discrepancy defined in
Section~\ref{subsection: discrepancy}. Conceptually, $\mathcal{DF}_{\theta}(S,T)$ measures a mismatch between
the gradient statistics induced by two (pseudo-) labeled sets under the current model. At acquisition round $t$,
the learner has a labeled set $\mathcal{D}_L^{(t)}$ and an unlabeled pool $\mathcal{D}_U^{(t)}$.
After training on $\mathcal{D}_L^{(t)}$ we obtain parameters $\theta_t$, which are then used to score candidates
in $\mathcal{D}_U^{(t)}$. Although the bound motivates reducing gradient discrepancy during training, in the active learning setting we use large discrepancy as a signal of under-representation: points whose gradients are far from the labeled-set reference are treated as candidates whose labels may help reduce this mismatch after acquisition and retraining.

\paragraph{Batch querying protocol (fair comparison).}
Labels are queried in batches: at round $t$, an acquisition rule selects a set
$\mathcal{B}^{(t)} \subset \mathcal{D}_U^{(t)}$ with $|\mathcal{B}^{(t)}|=b$ without access to oracle labels.
Only after the full batch is selected do we query the oracle once to obtain $\{y(x): x\in \mathcal{B}^{(t)}\}$.
Accordingly, whenever a label is required for scoring we use pseudo-labels
$\hat{y}_{\theta_t}(x)=\arg\max_y p_{\theta_t}(y\mid x)$.

\paragraph{$\mathcal{DF}$ score for a candidate point.}
For $x\in \mathcal{D}_U^{(t)}$, define the pseudo-labeled singleton
$\hat{S}_{\theta_t}(x)=\{(x,\hat{y}_{\theta_t}(x))\}$ and the augmented set
$S_x^{(t)}=\mathcal{D}_L^{(t)} \cup \hat{S}_{\theta_t}(x)$.
We score $x$ by the set-to-point discrepancy
\begin{equation}
s_t(x) \;=\; \left\| \mathcal{DF}_{\theta_t}\!\left(S_x^{(t)}, \hat{S}_{\theta_t}(x)\right) \right\|_2.
\label{eq:df_score}
\end{equation}
Intuitively, $s_t(x)$ is large when the training signal induced by the current labeled set differs from the
training signal induced by the candidate point (under its pseudo-label), and we treat such points as more
informative (see Algorithm~\ref{alg:grad}).
\paragraph{Relation to prior gradient-based acquisition.}
BADGE uses (pseudo-)labeled per-example gradient embeddings primarily to enforce within-batch diversity via clustering,
whereas we use gradients to define a \emph{reference-misalignment} score
$s_t(x)=\big\|\mathcal{DF}_{\theta_t}(S_x^{(t)},\hat{S}_{\theta_t}(x))\big\|_2$ and select the largest discrepancies.
EGL ranks points by the \emph{expected} gradient norm under the predictive distribution (requiring an expectation over labels),
while our score uses a single pseudo-label together with a reusable labeled reference statistic.

\begin{algorithm}[t]
\caption{Active learning with batch $\mathcal{DF}$ scoring (\texttt{grad})}
\label{alg:grad}
\begin{algorithmic}[1]
\State \textbf{Input:} labeled set $\mathcal{D}_L$, unlabeled pool $\mathcal{D}_U$, model $M$, oracle $O$, batch size $b$, rounds $n_r$
\For{$t=1$ \textbf{to} $n_r$}
    \State Train $M$ on $\mathcal{D}_L$ to obtain parameters $\theta_t$
    \ForAll{$x \in \mathcal{D}_U$}
        \State $\hat{y}\gets \hat{y}_{\theta_t}(x)$
        \State $\hat{S}(x)\gets \{(x,\hat{y})\}$
        \State $s_t(x)\gets \left\| \mathcal{DF}_{\theta_t}\!\left(\mathcal{D}_L \cup \hat{S}(x), \hat{S}(x)\right) \right\|_2$
        \Comment{implemented efficiently via a maintained reference statistic (Sec.~\ref{subsec:df_eff})}
    \EndFor
    \State $\mathcal{B} \gets$ top-$b$ points in $\mathcal{D}_U$ by $s_t(\cdot)$
    \State Query oracle labels $\{y(x):x\in\mathcal{B}\}$ from $O$
    \State $\mathcal{D}_L \gets \mathcal{D}_L \cup \{(x,y(x)):x\in\mathcal{B}\}$
    \State $\mathcal{D}_U \gets \mathcal{D}_U \setminus \mathcal{B}$
\EndFor
\end{algorithmic}
\end{algorithm}

\subsection{Efficient Computation of the \texorpdfstring{$\mathcal{DF}$}{DF} Score}
\label{subsec:df_eff}

Computing $s_t(x)=\left\|\mathcal{DF}_{\theta_t}\!\left(R\cup \hat{S}_{\theta_t}(x),\hat{S}_{\theta_t}(x)\right)\right\|_2$
does not require explicitly forming $R\cup \hat{S}_{\theta_t}(x)$ and recomputing full-set gradients for every candidate.
For the standard empirical-risk objective
$f(\theta;S)=|S|^{-1}\sum_{(x,y)\in S}\ell(\theta;x,y)$, the $\mathcal{DF}$ score admits an implementation using
(i) a maintained summary statistic of the reference set $R$ (e.g., its mean gradient embedding at $\theta_t$) and
(ii) per-example gradient embeddings for candidates, yielding an equivalent ranking up to an $x$-independent constant factor.
For objectives with non-additive aggregation across samples, this simplification need not hold; in that case one can compute
the $\mathcal{DF}$ score using the corresponding objective-level gradients or maintain the appropriate reference statistic.
In practice we often restrict embeddings to the last layer and optionally estimate the reference statistic with mini-batches;
runtime implications are reported in our running-time analysis.

\paragraph{Practical computation via gradient embeddings (full $\theta$ vs.\ last layer).}
Our definition of $\mathcal{DF}_{\theta}(S,T)$ applies to the full parameter vector $\theta\in\mathbb{R}^d$.
In practice, we compute $\mathcal{DF}$ using per-example gradient embeddings and often restrict the embedding
to a subset of parameters (typically the final layer) for efficiency.

Concretely, for a (pseudo-)labeled example $(x,\tilde{y})$, define the per-example gradient
\begin{equation}
g_{\theta}(x,\tilde{y}) \;:=\; \nabla_{\theta}\,\ell(\theta;x,\tilde{y}).
\end{equation}
When using last-layer embeddings, let $\theta^{(L)}$ denote the parameters of the final layer and define
\begin{equation}
g^{(L)}_{\theta}(x,\tilde{y}) \;:=\; \nabla_{\theta^{(L)}}\,\ell(\theta;x,\tilde{y}).
\end{equation}
We then compute the reference statistic from the corresponding mean of per-example (last-layer) gradients over $R$
and compute candidate embeddings from per-example (last-layer) gradients for $\hat{S}_{\theta_t}(x)$.

\paragraph{Estimating mean gradients.}
The reference statistic for a set $R$ (e.g., a mean gradient embedding at $\theta_t$) can be computed either
(i) by a full pass over the set, or (ii) by a mini-batch estimator, depending on the trade-off between accuracy and runtime.
In our experiments, we use the same choice consistently across acquisition rounds.

\section{Experimental Results}
In this section, we conduct experiments to evaluate the performance of our proposed acquisition score. The first subsection covers diversity exploration and out of distribution experiments, while the second part focuses on the active learning experiment.

\subsection{Qualitative Sanity Checks: Acquisition Geometry and Dataset Shift}
\label{subsec:vis-ablation}
This subsection provides qualitative diagnostics to clarify what our $\mathcal{DF}$-based batch rules select and why.
We visualize (i) the geometry of one-step acquisitions in input space and in final-layer gradient space, and
(ii) whether $\mathcal{DF}$ responds to dataset shift. These checks are meant to support interpretation of later learning
curves, not to serve as a quantitative diversity analysis.

\paragraph{Acquisition geometry on MNIST.}
We subsample $N{=}3{,}000$ MNIST examples, flatten each $28{\times}28$ image, and draw an initial labeled set
$I_0$ of size $|I_0|{=}10$ uniformly at random; the remaining points form the unlabeled pool $P$.
We train a simple feed-forward network (one hidden layer, linear output) on $I_0$ with cross-entropy.
For gradient-based methods ($\mathcal{DF}$ and \textsc{BADGE}), we compute gradients with respect to the final linear layer.

To isolate the effect of the acquisition rule, we fix the same $(I_0,P)$ and the same trained model parameters,
and perform a single batch acquisition with $B \in \{10,20,40\}$ using: \texttt{grad} (top-$B$ by $\mathcal{DF}$), Entropy (uncertainty sampling), and
\textsc{BADGE} (k-means++ seeding in pseudo-labeled gradient space~\citep{Ash2020}). All methods use pseudo-labels
when needed for scoring.

\paragraph{Input- and gradient-space visualizations.}
We project (i) flattened inputs and (ii) pseudo-labeled final-layer gradient embeddings to 2D via PCA, and overlay
the pool $P$ (grey), the initial labeled set $I_0$ (black), and the acquired batch (colored); see
Figure~\ref{fig:acq-geometry}. A consistent qualitative pattern is that Entropy tends to concentrate selections in a
relatively small region, whereas $\mathcal{DF}$-based rule spread selections more broadly. In this sense, \texttt{grad} behaves
like an uncertainty proxy (it targets points that are strongly ``misaligned'' with the labeled-set gradient signal),
but unlike entropy-based uncertainty it does not collapse into a tightly bounded region in these projections.
In both input and gradient space, $\mathcal{DF}$-based selections typically appear more dispersed than Entropy.
\textsc{BADGE} similarly tends to cover multiple regions in gradient space, consistent with its explicit diversity
mechanism.

\begin{figure}[t]
    \centering
    \begin{minipage}[b]{0.32\textwidth}
        \centering
        \includegraphics[width=\textwidth]{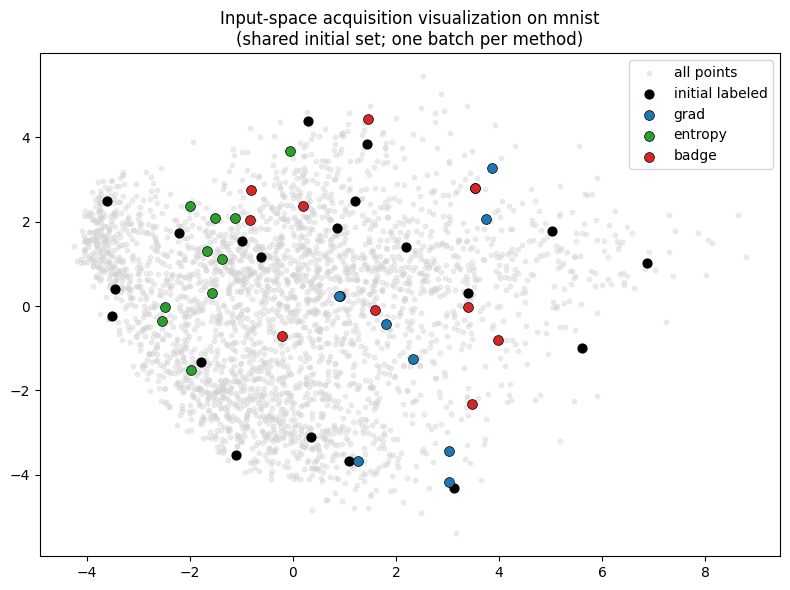}
        \caption*{Input space, $B{=}10$}
    \end{minipage}\hfill
    \begin{minipage}[b]{0.32\textwidth}
        \centering
        \includegraphics[width=\textwidth]{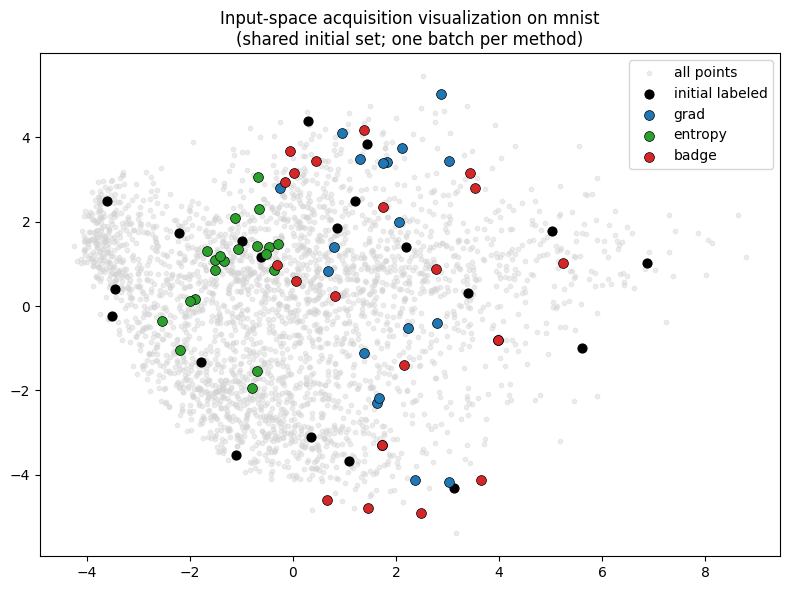}
        \caption*{Input space, $B{=}20$}
    \end{minipage}\hfill
    \begin{minipage}[b]{0.32\textwidth}
        \centering
        \includegraphics[width=\textwidth]{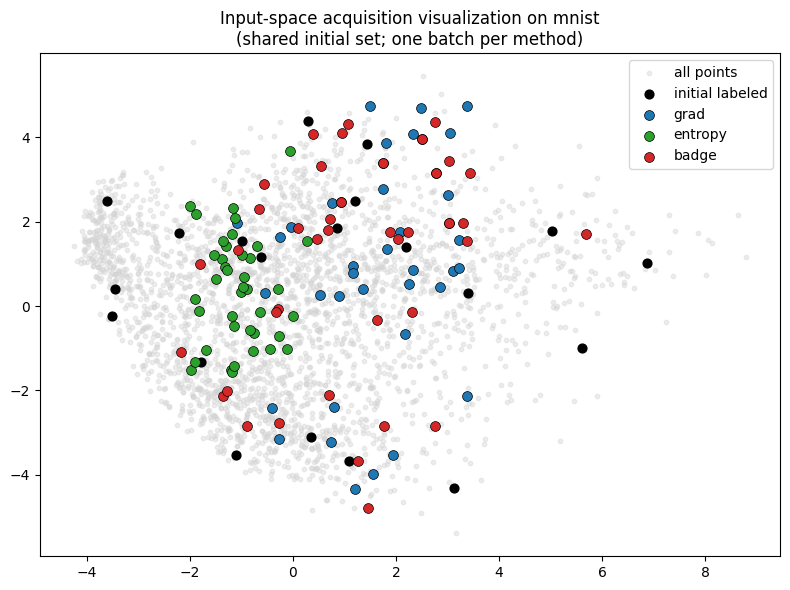}
        \caption*{Input space, $B{=}40$}
    \end{minipage}

    \vspace{0.35cm}

    \begin{minipage}[b]{0.32\textwidth}
        \centering
        \includegraphics[width=\textwidth]{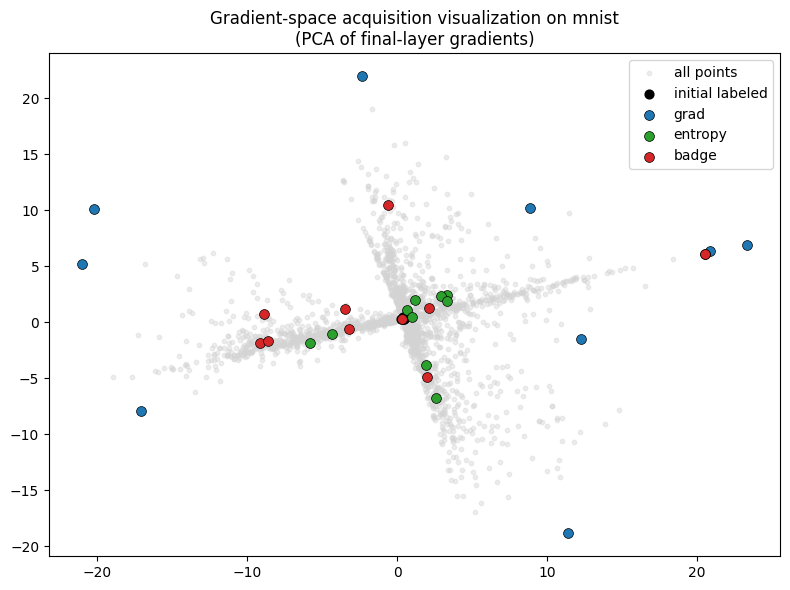}
        \caption*{Gradient space, $B{=}10$}
    \end{minipage}\hfill
    \begin{minipage}[b]{0.32\textwidth}
        \centering
        \includegraphics[width=\textwidth]{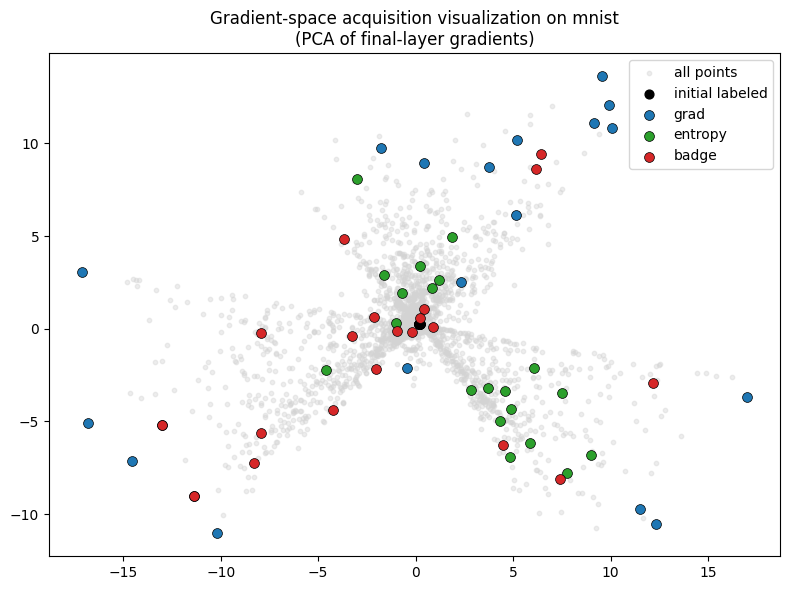}
        \caption*{Gradient space, $B{=}20$}
    \end{minipage}\hfill
    \begin{minipage}[b]{0.32\textwidth}
        \centering
        \includegraphics[width=\textwidth]{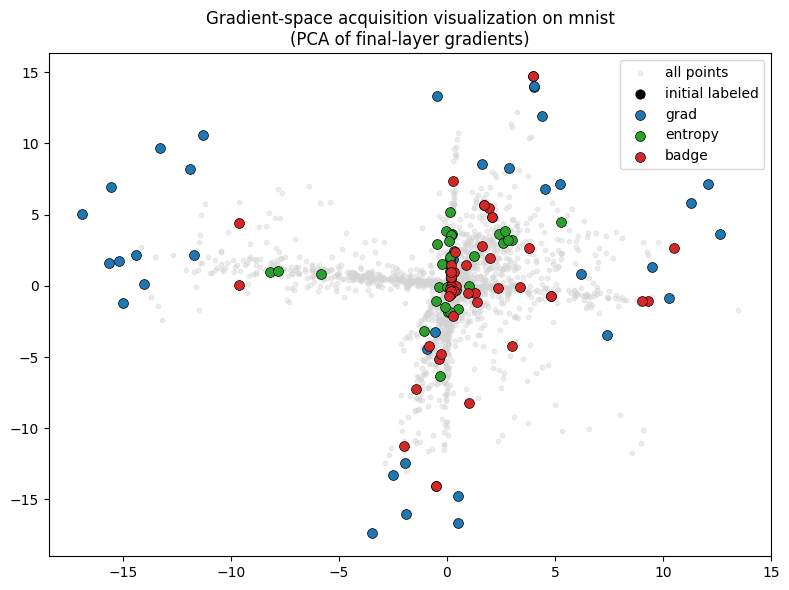}
        \caption*{Gradient space, $B{=}40$}
    \end{minipage}

    \caption{MNIST sanity check. Geometry of selected points in input space (top row) and final-layer gradient space (bottom row),
    using a shared initial labeled set and one acquisition step per method.}
    \label{fig:acq-geometry}
\end{figure}

\paragraph{\textbf{Dataset shift sensitivity}.}
We next probe whether $\mathcal{DF}$ changes under distribution shift. We train a model on the training set of CIFAR-10 and compute $\mathcal{DF}$ scores for two
evaluation sets of equal size ($n{=}10{,}000$): CIFAR-10 test (in-distribution) and SVHN test (out-of-distribution).
Figure~\ref{fig:df-shift} plots histograms of the resulting $\mathcal{DF}$ values. The out-of-distribution scores are shifted toward
larger values, suggesting that $\mathcal{DF}$ is sensitive to dataset mismatch in this setting and may therefore interact with robustness
and pseudo-label quality under shift.

\begin{figure}[t]
    \centering
    \includegraphics[width=0.62\textwidth]{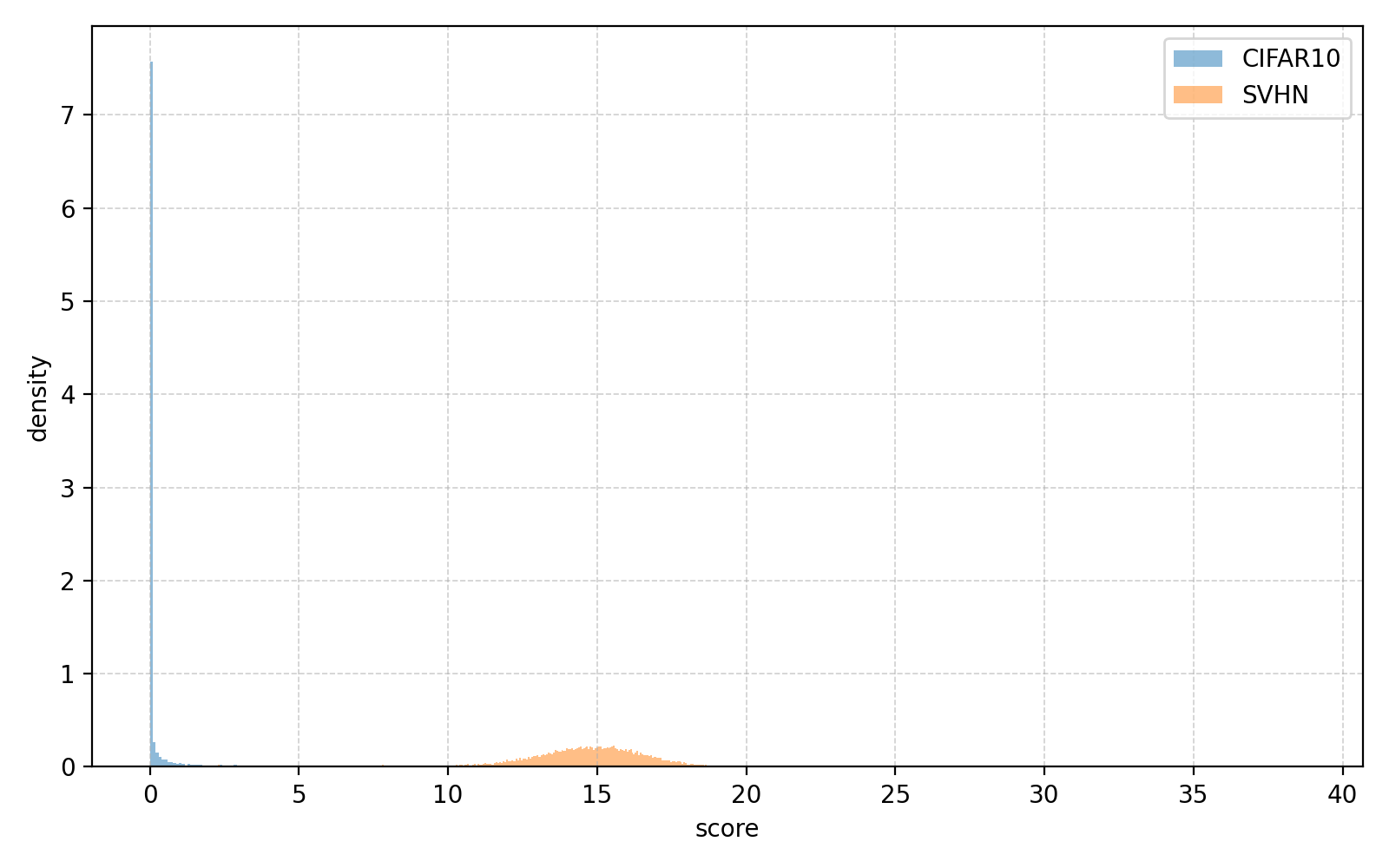}
    \caption{Histograms of $\mathcal{DF}$ scores on CIFAR-10 (in-distribution) and SVHN (out-of-distribution), each with $n{=}10{,}000$ examples.}
    \label{fig:df-shift}
\end{figure}

\subsection{Active Learning}
\label{sec:active_learning}

We consider the standard pool-based active learning (AL) protocol. At acquisition round $t$, given a labeled set $D_L^{(t)}$ and an unlabeled pool $D_U^{(t)}$, the learner selects a batch of $b$ unlabeled points, queries their labels, augments the labeled set, and retrains the model \emph{from scratch} on the updated labeled data. Unless stated otherwise, we report test accuracy as a function of acquisition round (equivalently, labeled budget). For fair comparison, within each random seed all methods share the same initial labeled set and the same unlabeled pool.

\paragraph{Compared methods.}
We compare our proposed gradient-based acquisition strategy (\textbf{grad} ) against four standard baselines: (i) uncertainty sampling via \emph{predictive entropy} under the current model; (ii) \emph{BADGE}, which forms a last-layer gradient embedding for each unlabeled point (using the model’s current predicted label) and selects a diverse batch via k-means++ seeding in this embedding space; (iii) \emph{random} sampling from the unlabeled pool; and (iv) \emph{K-center} selection, which greedily selects farthest-first points in the model’s representation space to encourage coverage~\citep{sener2017active}.

\subsection{Experimental Setting}
\label{sec:al_setting}

\paragraph{Datasets.}
We evaluate across modalities to reduce the likelihood that conclusions are driven by a single data type. We use:
(i) \textbf{text}: \textit{
20~Newsgroups}~\citep{lang1995newsweeder,qwone20newsgroups};
(ii) \textbf{tabular}: \textit{ISOLET} and \textit{OpenML 155 (pokerhand-normalized)}~\citep{uci_isolet_1991,uci_pokerhand_2002,vanschoren2013openml};
and (iii) \textbf{image}: \textit{CIFAR-10}, \textit{STL-10}, \textit{CINIC-10}, \textit{GTSRB}, and \textit{SVHN}~\citep{krizhevsky2009learning,stl10_dataset,darlow2018cinic10_dataset,stallkamp2012gtsrb,netzer2011svhn}.

\paragraph{Splits, pooling, and preprocessing.}
For 20~Newsgroups, we use the standard train/test split and represent documents using TF--IDF features with English stop-word removal after stripping headers, footers, and quoted text. For ISOLET and OpenML 155, we create a stratified 80/20 train/test split and standardize features using \texttt{StandardScaler}. For image datasets, we use official train/test splits for evaluation and form the unlabeled pool from the training split (optionally subsampled for computational efficiency). In particular, CIFAR-10 and SVHN use an unlabeled pool of 25{,}000 training examples obtained by uniform subsampling; SVHN uses the \texttt{train} split only (not \texttt{extra}). CINIC-10 uses the union of train and validation splits as the training source and is uniformly subsampled to 100{,}000 examples. All images are resized to $32\times 32$ and normalized with mean/std $(0.5,0.5,0.5)$.

\paragraph{Budgets and acquisition batch sizes.}
Let $|D_L^{(0)}|$ denote the initial labeled set size, $b$ the query batch size, and $T$ the number of acquisition rounds; the final labeled budget is $|D_L^{(T)}| = |D_L^{(0)}| + Tb$. We set $|D_L^{(0)}|=b$ throughout. For text/tabular datasets we use $(b,T)=(100,39)$. For image datasets we consider $(b,T)\in\{(100,14),(500,8),(1000,4)\}$ depending on dataset scale; the specific acquisition regime and backbone used in each experiment are indicated in the corresponding figure title/caption.

\paragraph{Model architectures.}
For text/tabular experiments, we use an MLP with two hidden layers of widths 512 and 256 followed by the final classification layer, trained end-to-end with backpropagation~\citep{rumelhart1986backprop}. For images, we use standard convolutional backbones (ResNet-18~\citep{he2016resnet}, VGG-16~\citep{simonyan2015vgg}, and LeNet~\citep{lecun1998gradient}) and replace the original classifier with a lightweight MLP head (512 and 256) before the final output layer.

\paragraph{\textbf{Training}.}
All models are implemented in PyTorch and trained with SGD (momentum 0.9, batch size 8) using cross-entropy loss.
For learning-rate selection, we create a fixed, stratified validation split from the training data once per dataset
(kept unchanged across methods and seeds), and we exclude this split from the unlabeled pool so it is never acquired.
We perform a one-time learning-rate sweep over $\{0.0001, 0.0005, 0.001, 0.005, 0.01\}$ and choose the rate with the
highest validation accuracy for the given dataset/backbone; this learning rate is then fixed for all acquisition rounds.
No learning-rate scheduling is used. The number of epochs is dataset-dependent (30--60). In the AL loop, models are
retrained from scratch after each acquisition round. No data augmentation is used. CUDA determinism is enabled to
improve run-to-run reproducibility.

\paragraph{Reporting across seeds.}
Each experiment is repeated with five random seeds. Curves report the mean test accuracy across seeds at each acquisition round; shaded regions indicate $\pm 1$ standard deviation when shown.

\subsection{Learning Curves}
\label{sec:al_learning_curves}

\begin{figure}[htbp]
    \centering

    \begin{minipage}[b]{0.49\textwidth}
        \centering
        \includegraphics[width=\linewidth]{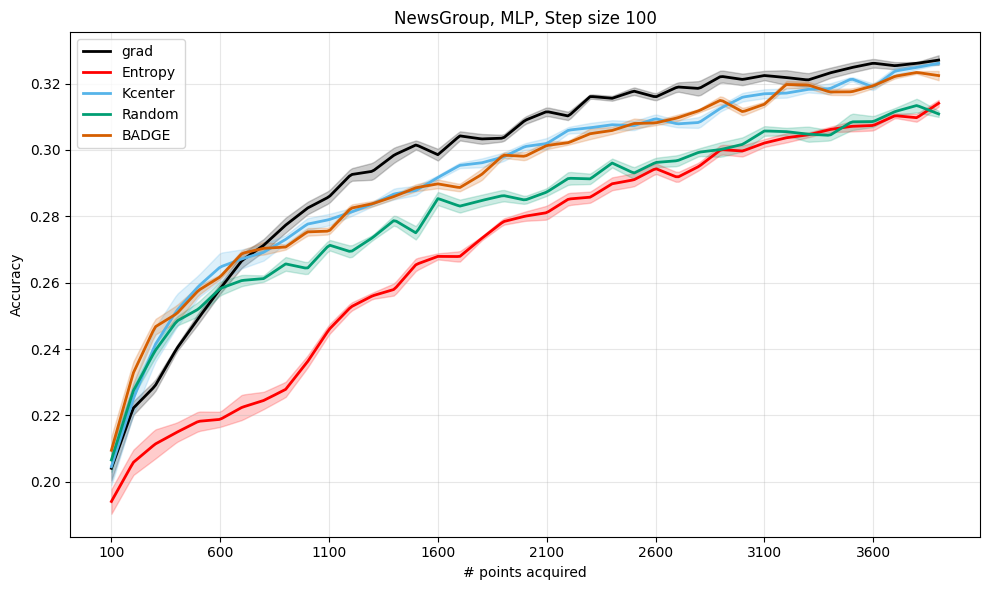}
        \caption*{(a)}
    \end{minipage}
    \hfill
    \begin{minipage}[b]{0.49\textwidth}
        \centering
        \includegraphics[width=\linewidth]{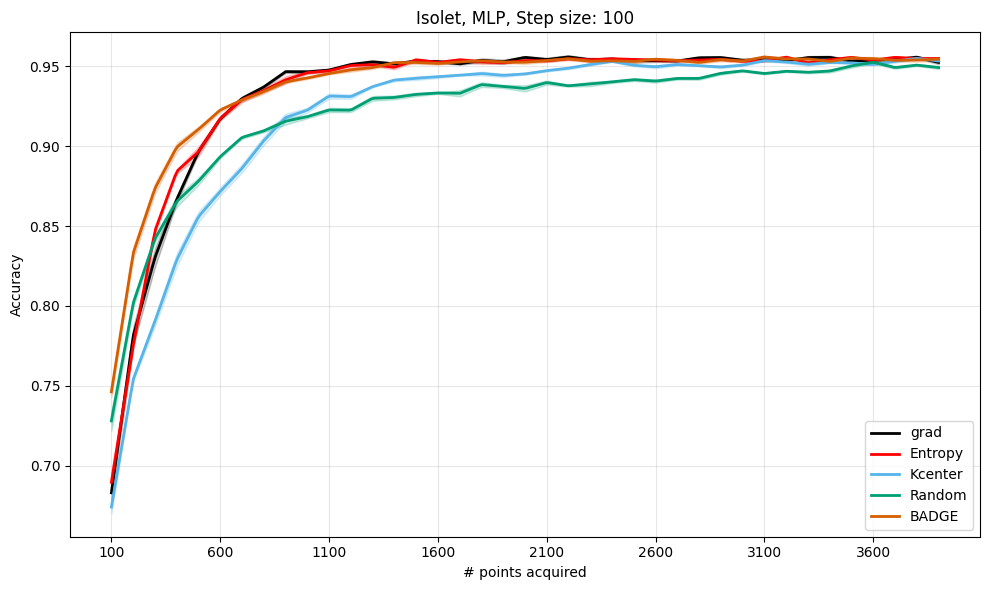}
        \caption*{(b)}
    \end{minipage}

    \vspace{0.5em}

    \begin{minipage}[b]{0.49\textwidth}
        \centering
        \includegraphics[width=\linewidth]{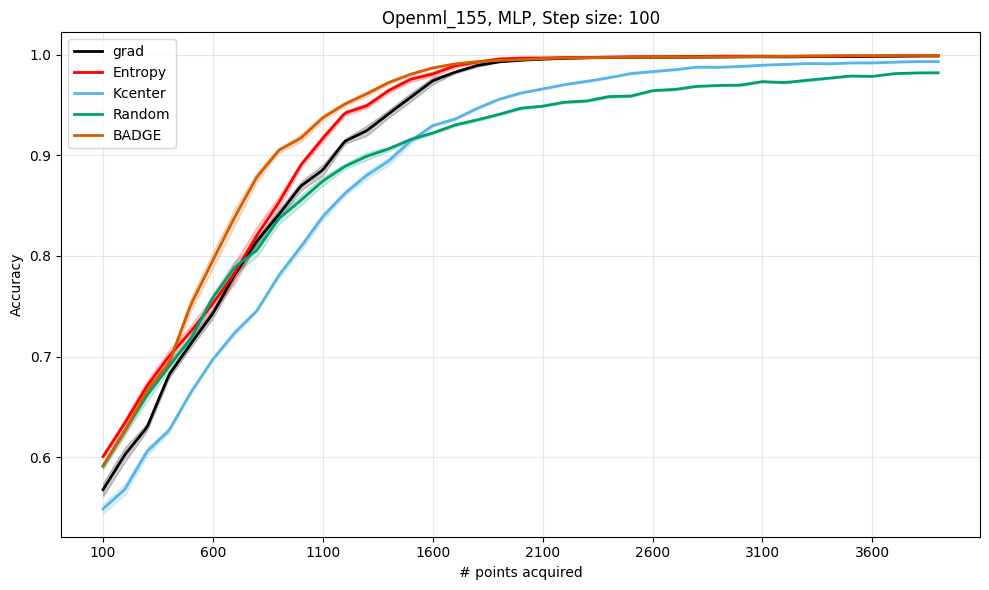}
        \caption*{(c)}
    \end{minipage}

    \caption{Active learning test accuracy, mean over 5 seeds, for text/tabular benchmarks with step size 100: 20 Newsgroups, ISOLET, and OpenML 155.}
    \label{fig:al_text_tabular}
\end{figure}

\paragraph{Text and tabular benchmarks.}
Figure~\ref{fig:al_text_tabular}  summarizes the active learning trajectories across three datasets. On 20~Newsgroups, the methods begin from comparable accuracy, but the gap becomes clearer as the budget increases: \texttt{grad} is typically among the most label-efficient methods, with K-center and BADGE remaining competitive, while Random and Entropy are generally less competitive. On ISOLET, performance differences are most visible in early-to-mid rounds; as labeling progresses, the curves converge, suggesting diminishing returns once sufficient labeled coverage is reached. On OpenML 155, acquisition choice has the largest impact in the mid-budget regime: BADGE and entropy-based uncertainty sampling often improve faster early on, while \textit{grad} remains consistently strong; at higher budgets, most informed strategies approach similar near-saturated accuracy, whereas Random continues to trail.

\begin{figure}[htbp]
    \centering

    \begin{minipage}[b]{0.48\textwidth}
        \centering
        \includegraphics[width=\linewidth, trim=20 10 20 10]{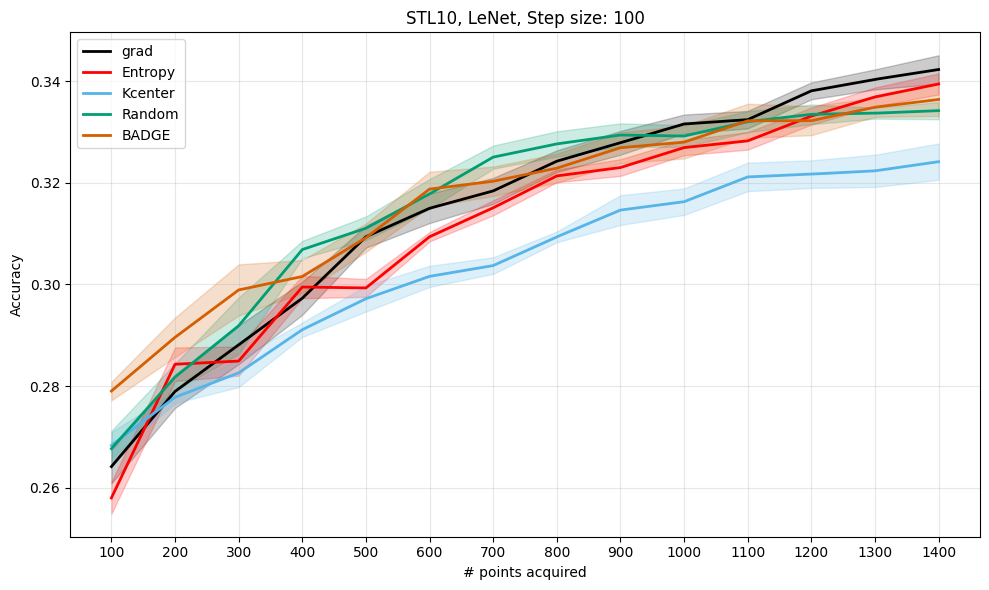}
        \caption*{(a)}
    \end{minipage}
    \hfill
    \begin{minipage}[b]{0.48\textwidth}
        \centering
        \includegraphics[width=\linewidth, trim=20 10 20 10]{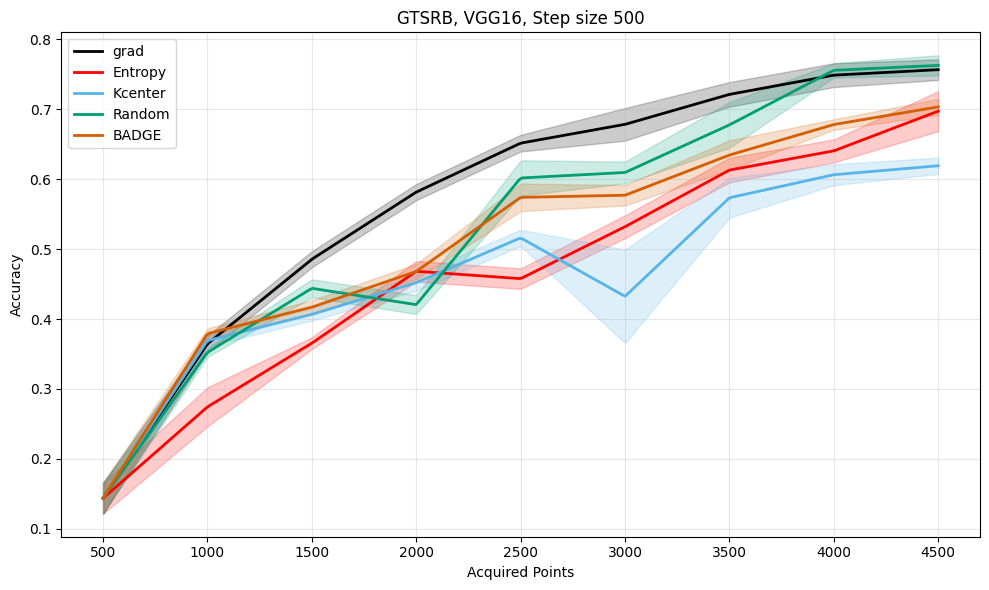}
        \caption*{(b)}
    \end{minipage}

    \vspace{0.4cm}

    \begin{minipage}[b]{0.48\textwidth}
        \centering
        \includegraphics[width=\linewidth, trim=20 10 20 10]{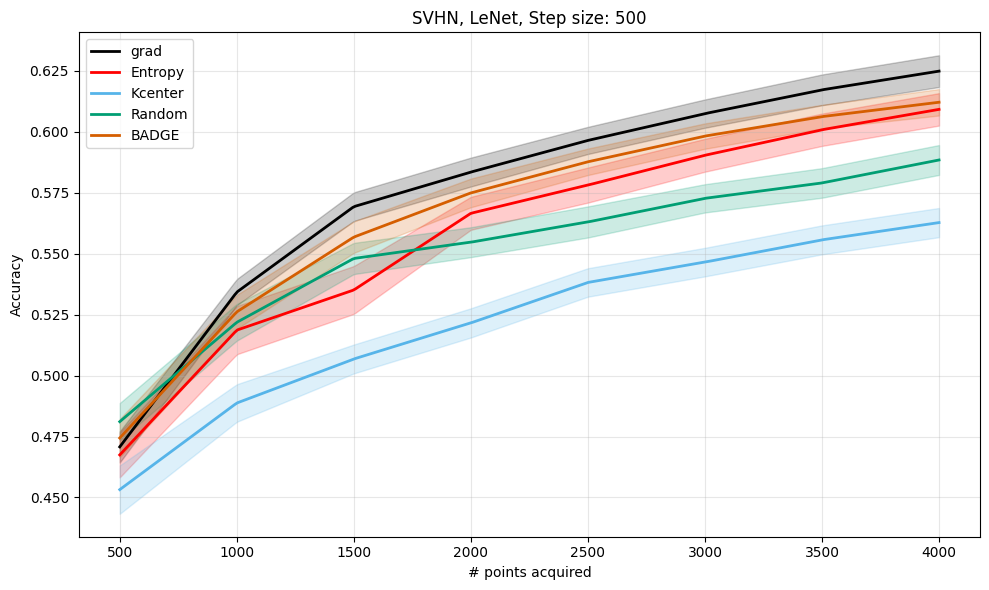}
        \caption*{(c)}
    \end{minipage}
    \hfill
    \begin{minipage}[b]{0.48\textwidth}
        \centering
        \includegraphics[width=\linewidth, trim=20 10 20 10]{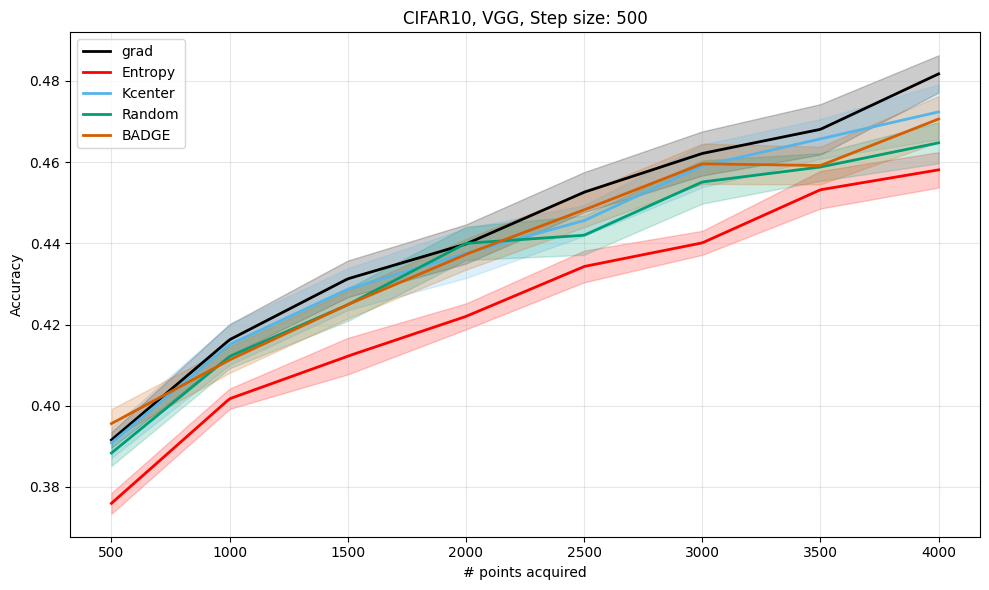}
        \caption*{(d)}
    \end{minipage}

    \vspace{0.4cm}

    \begin{minipage}[b]{0.48\textwidth}
        \centering
        \includegraphics[width=\linewidth, trim=20 10 20 10]{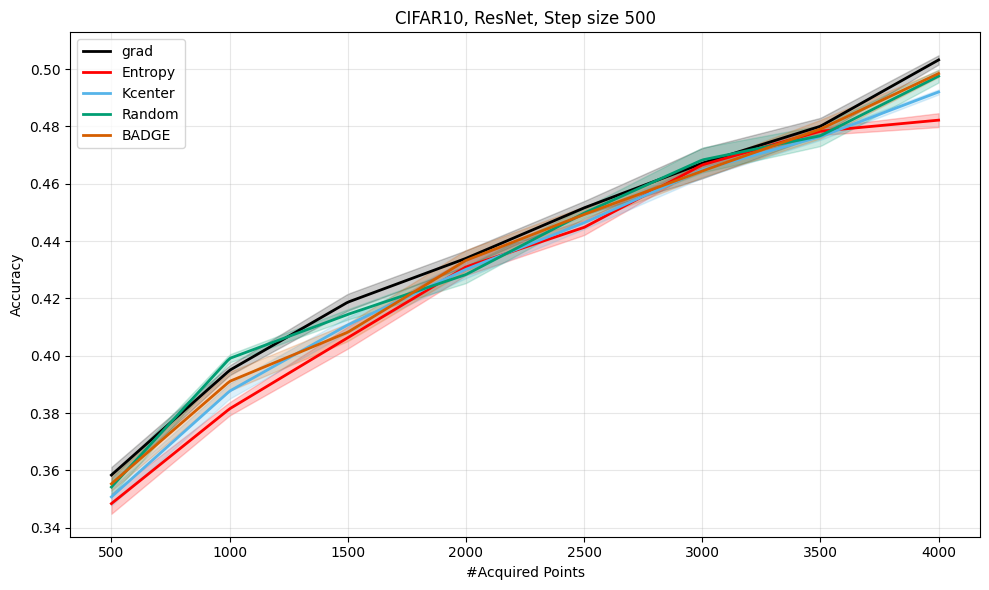}
        \caption*{(e)}
    \end{minipage}
    \hfill
    \begin{minipage}[b]{0.48\textwidth}
        \centering
        \includegraphics[width=\linewidth, trim=20 10 20 10]{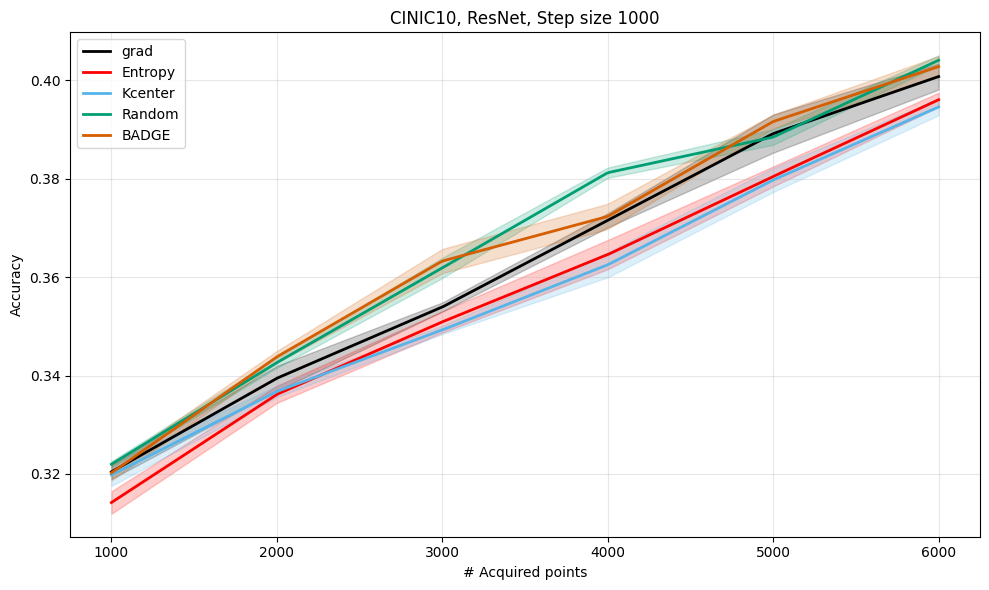}
        \caption*{(f)}
    \end{minipage}

    \caption{Active learning test accuracy on image benchmarks, mean over five seeds; shaded regions indicate variability.}
    \label{fig:al_cifar_svhn_500}
\end{figure}

\paragraph{Image benchmarks.}
Figure~\ref{fig:al_cifar_svhn_500} provides representative learning curves across datasets and backbones.

\noindent\textbf{(a) STL-10 (LeNet).}
All methods improve steadily and remain close. BADGE is slightly stronger early on, consistent with the benefit of selecting a diverse initial batch, while grad becomes more competitive at later rounds.

\noindent\textbf{(b) GTSRB (VGG-16).}
Gradient-based acquisition (grad) is consistently strong and reaches high accuracy quickly. Random and BADGE track closely at later budgets, whereas K-center is less stable, with a noticeable mid-budget dip before recovering.

\noindent\textbf{(c) SVHN (LeNet).}
BADGE and grad remain competitive, while Random and Entropy improve more gradually. K-center is consistently weaker, suggesting that diversity alone is not sufficient in this setting.

\noindent\textbf{(d) CIFAR-10 (VGG-16).}
\texttt{grad} is often among the strongest methods, with BADGE and K-center close behind. Entropy lags for much of the trajectory, although the gap narrows as the labeled set increases.

\noindent\textbf{(e) CIFAR-10 (ResNet).}
Methods are tightly clustered across most budgets. Grad is usually marginally ahead, while Entropy is slightly behind; differences become small at higher budgets.

\noindent\textbf{(f) CINIC-10 (ResNet).}
The same pattern persists: gradient- and diversity-aware strategies provide modest but consistent early gains, while performance converges as the labeled set grows.

Additional learning curves for other backbones and acquisition schedules are included in
Appendix~\ref{app:al_curves}. Since individual learning curves can emphasize different parts of the budget
range, we next provide an aggregate comparison across all experiments and acquisition rounds.

\subsection{Overall Comparison}
\label{sec:overall}
\paragraph{Pairwise penalty matrix.}
We compare all query strategies across a suite of experimental settings (datasets, architectures, and acquisition schedules). For each experiment and each active learning round, we run every method with five random seeds and record test accuracy. Within a fixed round of a fixed experiment, we then compare each method pair $(i,j)$ using a paired $t$-test on the per-seed accuracies. To control for the multiple pairwise comparisons performed within each round, we apply a Benjamini--Hochberg false discovery rate (FDR) correction to the resulting $p$-values and declare a win only if the adjusted $p$-value falls below $\alpha$ (we use $\alpha=0.05$). If $i$ significantly outperforms $j$, we add a contribution of $1/n_e$ to $P_{i,j}$, where $n_e$ is the number of evaluated rounds in that experiment; if $j$ significantly outperforms $i$, we add $1/n_e$ to $P_{j,i}$. This normalization ensures that each experiment contributes at most one unit of win mass to each ordered pair, independent of how many rounds it contains. Aggregating over all experiments yields a pairwise penalty matrix $P \in \mathbb{R}^{K\times K}$, where $K$ is the number of methods.

This construction follows standard guidance for multi-dataset algorithm comparison via systematic statistical testing and pairwise post-hoc comparisons~\citep{Demsar2006,Garcia2008,Dolan2002}, and matches the PPM-style evaluation commonly used in deep active learning~\citep{Ash2020}.

\paragraph{Interpretation.}
The entry $P_{i,j}$ measures how often method $i$ significantly outperforms method $j$ across the benchmark suite (larger values indicate that $i$ more frequently beats $j$). Rows therefore summarize wins and columns summarize losses. To obtain a single scalar score per method, we compute a loss score as the column-wise average of $P$; methods with lower loss scores are less frequently beaten and are thus stronger overall. The bar plots visualize this ranking, complementing the more detailed pairwise structure in the heatmaps.

\paragraph{Results.}
Figure~\ref{fig:ppm_overall} reports three aggregated views: over all rounds, over the early-stage rounds, and over the late-stage rounds. Across all aggregations, grad and BADGE  achieve the strongest overall performance. K-center is generally intermediate, while Entropy and Random are most frequently beaten. The early-stage aggregation highlights BADGE as particularly competitive when labels are scarce, whereas the late-stage aggregation emphasizes the advantage of gradient-based selection as training progresses and methods approach saturation.

\begin{figure}[htbp]
    \centering
    \begin{minipage}[b]{0.32\textwidth}
        \centering
        \includegraphics[width=\textwidth]{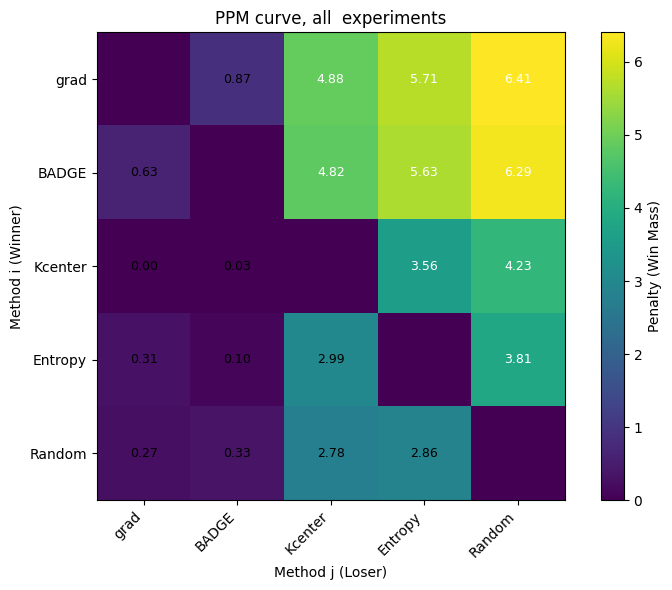}
        \caption*{(a)}
    \end{minipage}\hfill
    \begin{minipage}[b]{0.32\textwidth}
        \centering
        \includegraphics[width=\textwidth]{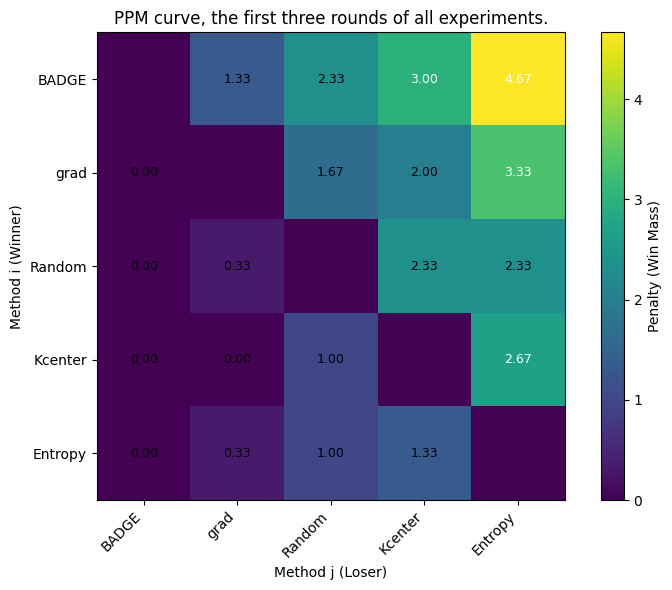}
        \caption*{(b)}
    \end{minipage}\hfill
    \begin{minipage}[b]{0.32\textwidth}
        \centering
        \includegraphics[width=\textwidth]{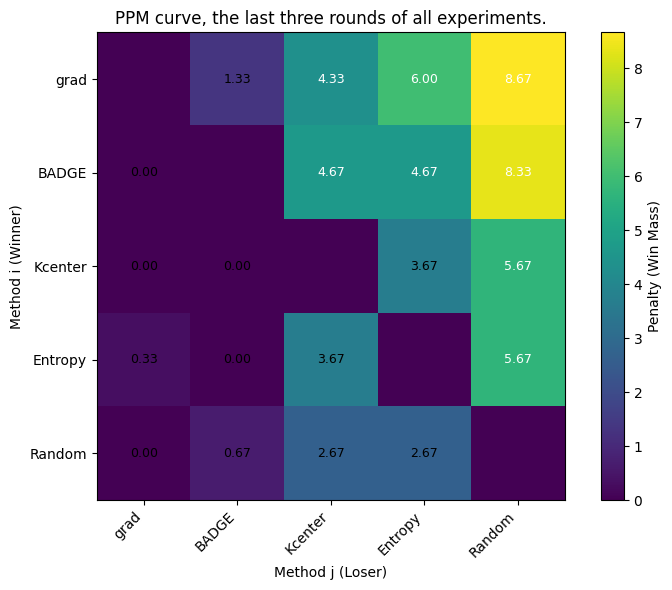}
        \caption*{(c)}
    \end{minipage}

    \vspace{0.4cm}

    \begin{minipage}[b]{0.32\textwidth}
        \centering
        \includegraphics[width=\textwidth]{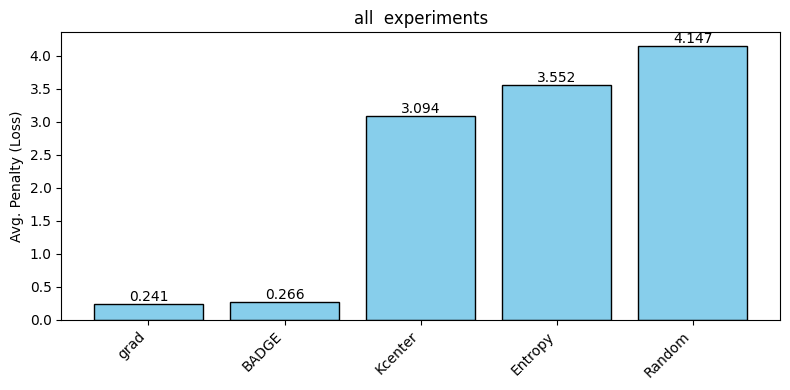}
        \caption*{(d)}
    \end{minipage}\hfill
    \begin{minipage}[b]{0.32\textwidth}
        \centering
        \includegraphics[width=\textwidth]{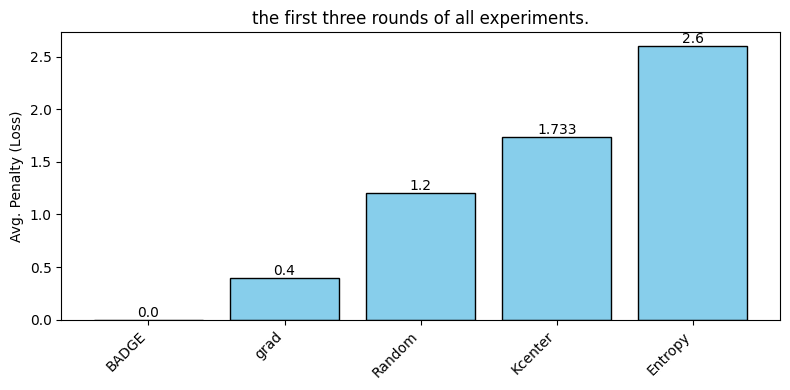}
        \caption*{(e)}
    \end{minipage}\hfill
    \begin{minipage}[b]{0.32\textwidth}
        \centering
        \includegraphics[width=\textwidth]{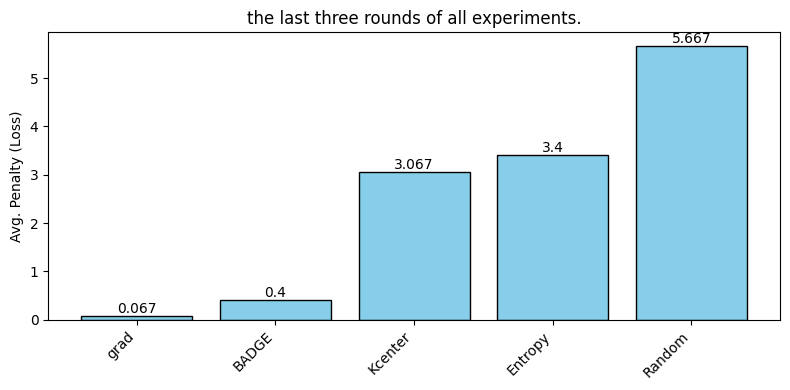}
        \caption*{(f)}
    \end{minipage}

    \caption{Overall comparison using the pairwise penalty matrix (top row) and the corresponding loss-score ranking (bottom row). (a,d) aggregate over all rounds; (b,e) early-stage rounds; (c,f) late-stage rounds. Larger PPM entries indicate more frequent statistically significant wins, while lower loss scores indicate stronger overall performance.}
    \label{fig:ppm_overall}
\end{figure}
\subsection{Computational Cost of the Acquisition Functions}
\label{subsec:acq-complexity}

We compare the cost of one acquisition round for the methods used in our experiments. This cost includes only
the acquisition step, i.e., selecting a batch of examples from the unlabeled pool. It does not include model
training or retraining.

In deep active learning, much of the acquisition cost comes from computing model outputs or embeddings for the
unlabeled pool. Entropy is relatively cheap because it only requires forward predictions. Gradient-based methods,
such as \texttt{grad} and \textsc{BADGE}, are more expensive because they require gradient embeddings. K-center
uses feature embeddings and then performs distance-based selection in feature space.

After the embeddings are computed, the methods differ in how they select the batch. Entropy simply ranks points
by predictive uncertainty. Our \texttt{grad} method compares each pool point with a labeled-set reference gradient
and selects the points with the largest discrepancy. In contrast, \textsc{BADGE} uses a k-means++ style selection
step in gradient space to encourage diversity within the selected batch. K-center repeatedly updates distances
from pool points to the labeled set and to already selected points, so its cost can increase as the labeled set
becomes larger.

\begin{table}[t]
\small
\centering
\begin{tabular}{|l|c|}
\hline
\textbf{Method} & \textbf{Time per round (s)} \\
\hline
Entropy & $1.06 \pm 0.05$ \\
\hline
\texttt{grad} (DF batch) & $2.44 \pm 0.31$ \\
\hline
\textsc{BADGE}~\citep{Ash2020} & $3.07 \pm 0.12$ \\
\hline
K-center (core-set)~\citep{sener2017active} & $3.47 \pm 0.23$ \\
\hline
\end{tabular}
\caption{Acquisition time per round, excluding model training. Times are mean $\pm$ standard deviation over
acquisition rounds 01--05 with an unlabeled pool of 25{,}000 examples and batch size 500. Embeddings are cached
within each acquisition round.}
\label{tab:acq-complexity-and-time}
\end{table}

The timing results are consistent with this computational picture. Entropy is the fastest method because it only
uses forward predictions. The \texttt{grad} method is slower than entropy because it uses gradient embeddings, but
it is faster than \textsc{BADGE} in this experiment because its batch selection step is simpler. \textsc{BADGE}
adds an explicit diversity step through k-means++ selection in gradient space, while \texttt{grad} performs direct
scoring followed by top-b selection. K-center is also relatively expensive because its distance computations grow
with the labeled set and the selected batch.
\section{Conclusion}
We introduced a gradient-discrepancy (DF) acquisition criterion for pool-based active learning, motivated by a discrepancy
term that appears in gradient-based generalization analysis. In implementation, DF uses pseudo-labels together with
gradient embeddings (e.g., from the last layer) to score unlabeled candidates by how different they are from a labeled-set
reference summary. This yields a bound-inspired selection rule that is practical to compute and does not require
explicit clustering.

Empirically, DF-based selection is competitive with common baselines in our benchmark suite;
in our experiments, \texttt{grad} typically remains competitive with \textsc{BADGE} overall. 

One limitation of the current \texttt{grad} rule is that it scores points independently, so the selected batch may contain redundant examples with similar gradient signals. Future work could incorporate an additional diversity check based on the angle or cosine similarity between last-layer gradient embeddings, rather than relying only on the discrepancy magnitude. This could help select batches that are both highly discrepant from the labeled-set reference and diverse within the batch.

\bibliography{main}

@article{luo2022generalization,
  title={Generalization bounds for gradient methods via discrete and continuous prior},
  author={Luo, Xuanyuan and Luo, Bei and Li, Jian},
  journal={Advances in Neural Information Processing Systems},
  volume={35},
  pages={10600--10614},
  year={2022}
}

@techreport{settles2009active,
  author      = {Settles, Burr},
  title       = {Active Learning Literature Survey},
  institution = {University of Wisconsin--Madison},
  type        = {Computer Sciences Technical Report},
  number      = {1648},
  year        = {2009}
}

@inproceedings{NIPS2007_a1519de5,
 author = {Settles, Burr and Craven, Mark and Ray, Soumya},
 booktitle = {Advances in Neural Information Processing Systems},
 editor = {J. Platt and D. Koller and Y. Singer and S. Roweis},
 pages = {},
 publisher = {Curran Associates, Inc.},
 title = {Multiple-Instance Active Learning},
 url = {https://proceedings.neurips.cc/paper_files/paper/2007/file/a1519de5b5d44b31a01de013b9b51a80-Paper.pdf},
 volume = {20},
 year = {2007}
}

@article{DBLP:journals/corr/abs-2103-00123,
  author       = {KrishnaTeja Killamsetty and
                  Durga Sivasubramanian and
                  Baharan Mirzasoleiman and
                  Ganesh Ramakrishnan and
                  Abir De and
                  Rishabh K. Iyer},
  title        = {{GRAD-MATCH:} {A} Gradient Matching Based Data Subset Selection for
                  Efficient Learning},
  journal      = {CoRR},
  volume       = {abs/2103.00123},
  year         = {2021},
  url          = {https://arxiv.org/abs/2103.00123},
  eprinttype    = {arXiv},
  eprint       = {2103.00123},
  bibsource    = {dblp computer science bibliography, https://dblp.org}
}

@article{DBLP:journals/corr/abs-2107-07075,
  author       = {Mansheej Paul and
                  Surya Ganguli and
                  Gintare Karolina Dziugaite},
  title        = {Deep Learning on a Data Diet: Finding Important Examples Early in
                  Training},
  journal      = {CoRR},
  volume       = {abs/2107.07075},
  year         = {2021},
  url          = {https://arxiv.org/abs/2107.07075},
  eprinttype    = {arXiv},
  eprint       = {2107.07075},
  bibsource    = {dblp computer science bibliography, https://dblp.org}
}

@techreport{krizhevsky2009learning,
  title={Learning multiple layers of features from tiny images},
  author={Krizhevsky, Alex},
  year={2009},
  institution={University of Toronto}
}

@article{lecun1998gradient,
  title={Gradient-based learning applied to document recognition},
  author={LeCun, Yann and Bottou, L{\'e}on and Bengio, Yoshua and Haffner, Patrick},
  journal={Proceedings of the IEEE},
  volume={86},
  number={11},
  pages={2278--2324},
  year={1998},
  publisher={IEEE}
}

@misc{gisette,
  author = {UCI Machine Learning Repository},
  title = {Gisette Dataset},
  year = {2003},
  url = {https://archive.ics.uci.edu/ml/datasets/gisette},
  note = {Accessed: 2025-05-28}
}

@inproceedings{gal2017deep,
  title={Deep bayesian active learning with image data},
  author={Gal, Yarin and Islam, Riashat and Ghahramani, Zoubin},
  booktitle={International conference on machine learning},
  pages={1183--1192},
  year={2017},
  organization={PMLR}
}

@article{sener2017active,
  title={Active learning for convolutional neural networks: A core-set approach},
  author={Sener, Ozan and Savarese, Silvio},
  journal={arXiv preprint arXiv:1708.00489},
  year={2017}
}

@article{Demsar2006,
  author  = {Janez Dem{\v{s}}ar},
  title   = {Statistical Comparisons of Classifiers over Multiple Data Sets},
  journal = {Journal of Machine Learning Research},
  volume  = {7},
  pages   = {1--30},
  year    = {2006}
}

@article{Garcia2008,
  author  = {Salvador Garc\'{i}a and Francisco Herrera},
  title   = {An Extension on ``Statistical Comparisons of Classifiers over Multiple Data Sets'' for all Pairwise Comparisons},
  journal = {Journal of Machine Learning Research},
  volume  = {9},
  pages   = {2677--2694},
  year    = {2008}
}

@article{Dolan2002,
  author  = {Elizabeth D. Dolan and Jorge J. Mor\'{e}},
  title   = {Benchmarking Optimization Software with Performance Profiles},
  journal = {Mathematical Programming},
  volume  = {91},
  number  = {2},
  pages   = {201--213},
  year    = {2002},
  doi     = {10.1007/s101070100263}
}

@inproceedings{Ash2020,
  author    = {Jordan T. Ash and Chicheng Zhang and Akshay Krishnamurthy and John Langford and Alekh Agarwal},
  title     = {Deep Batch Active Learning by Diverse, Uncertain Gradient Lower Bounds},
  booktitle = {Proceedings of the International Conference on Learning Representations (ICLR)},
  year      = {2020},
  url       = {https://openreview.net/forum?id=ryghZJBKPS}
}

@inproceedings{lang1995newsweeder,
  author    = {Lang, Ken},
  title     = {NewsWeeder: Learning to Filter Netnews},
  booktitle = {Proceedings of the Twelfth International Conference on Machine Learning (ICML)},
  pages     = {331--339},
  year      = {1995}
}

@misc{qwone20newsgroups,
  author       = {Rennie, Jason},
  title        = {20 Newsgroups Data Set},
  year         = {2008},
  howpublished = {\url{https://qwone.com/~jason/20Newsgroups/}},
  note         = {Accessed 2025-12-15}
}

@misc{uci_pokerhand_2002,
  author       = {Cattral, Ryan and Oppacher, Franz},
  title        = {Poker Hand [Dataset]},
  howpublished = {UCI Machine Learning Repository},
  year         = {2002},
  doi          = {10.24432/C5KW38},
  url          = {https://archive.ics.uci.edu/ml/datasets/poker+hand},
  note         = {Accessed 2025-12-15}
}

@misc{uci_isolet_1991,
  author       = {Cole, Ron and Fanty, Mark},
  title        = {ISOLET [Dataset]},
  howpublished = {UCI Machine Learning Repository},
  year         = {1991},
  doi          = {10.24432/C51G69},
  url          = {https://archive.ics.uci.edu/ml/datasets/isolet},
  note         = {Accessed 2025-12-15}
}

@misc{stl10_dataset,
  author       = {Coates, Adam and Lee, Honglak and Ng, Andrew Y.},
  title        = {STL-10 Dataset},
  year         = {2011},
  howpublished = {Stanford University},
  url          = {http://cs.stanford.edu/~acoates/stl10},
  note         = {Accessed 2025-12-15}
}

@misc{darlow2018cinic10_dataset,
  author    = {Darlow, Luke N. and Crowley, Elliot J. and Antoniou, Antreas and Storkey, Amos},
  title     = {CINIC-10 Is Not ImageNet or CIFAR-10 [Dataset]},
  year      = {2018},
  publisher = {University of Edinburgh},
  doi       = {10.7488/ds/2448},
  url       = {https://datashare.ed.ac.uk/handle/10283/3192},
  note      = {Accessed 2025-12-15}
}

@article{stallkamp2012gtsrb,
  author  = {Stallkamp, Johannes and Schlipsing, Marc and Salmen, Jan and Igel, Christian},
  title   = {Man vs. Computer: Benchmarking Machine Learning Algorithms for Traffic Sign Recognition},
  journal = {Neural Networks},
  volume  = {32},
  pages   = {323--332},
  year    = {2012},
  doi     = {10.1016/j.neunet.2012.02.016}
}

@inproceedings{netzer2011svhn,
  author    = {Netzer, Yuval and Wang, Tao and Coates, Adam and Bissacco, Alessandro and Wu, Bo and Ng, Andrew Y.},
  title     = {Reading Digits in Natural Images with Unsupervised Feature Learning},
  booktitle = {NIPS Workshop on Deep Learning and Unsupervised Feature Learning},
  year      = {2011},
  url       = {http://ufldl.stanford.edu/housenumbers/}
}

@article{vanschoren2013openml,
  author  = {Vanschoren, Joaquin and van Rijn, Jan N. and Bischl, Bernd and Torgo, Luis},
  title   = {OpenML: Networked Science in Machine Learning},
  journal = {SIGKDD Explorations},
  volume  = {15},
  number  = {2},
  pages   = {49--60},
  year    = {2013},
  doi     = {10.1145/2641190.2641198}
}

@article{rumelhart1986backprop,
  author  = {Rumelhart, David E. and Hinton, Geoffrey E. and Williams, Ronald J.},
  title   = {Learning representations by back-propagating errors},
  journal = {Nature},
  volume  = {323},
  pages   = {533--536},
  year    = {1986},
  doi     = {10.1038/323533a0},
  url     = {https://www.nature.com/articles/323533a0}
}

@inproceedings{he2016resnet,
  author    = {He, Kaiming and Zhang, Xiangyu and Ren, Shaoqing and Sun, Jian},
  title     = {Deep Residual Learning for Image Recognition},
  booktitle = {2016 IEEE Conference on Computer Vision and Pattern Recognition (CVPR)},
  pages     = {770--778},
  year      = {2016},
  doi       = {10.1109/CVPR.2016.90},
  url       = {https://www.cv-foundation.org/openaccess/content_cvpr_2016/papers/He_Deep_Residual_Learning_CVPR_2016_paper.pdf}
}

@inproceedings{simonyan2015vgg,
  author    = {Simonyan, Karen and Zisserman, Andrew},
  title     = {Very Deep Convolutional Networks for Large-Scale Image Recognition},
  booktitle = {International Conference on Learning Representations (ICLR)},
  year      = {2015},
  url       = {https://arxiv.org/abs/1409.1556}
}

@inproceedings{lowell2019practical,
  title={Practical Obstacles to Deploying Active Learning},
  author={Lowell, David and Lipton, Zachary C. and Wallace, Byron C.},
  booktitle={Proceedings of the 2019 Conference on Empirical Methods in Natural Language Processing and the 9th International Joint Conference on Natural Language Processing (EMNLP-IJCNLP)},
  pages={21--30},
  year={2019}
}

@inproceedings{ovadia2019can,
  title={Can You Trust Your Model's Uncertainty? Evaluating Predictive Uncertainty Under Dataset Shift},
  author={Ovadia, Yaniv and Fertig, Emily and Ren, Jie and Nado, Zachary and Sculley, D. and Nowozin, Sebastian and Dillon, Joshua V. and Lakshminarayanan, Balaji and Snoek, Jasper},
  booktitle={Advances in Neural Information Processing Systems},
  year={2019}
}
\bibliographystyle{tmlr}

\appendix

\section{Contraction of Gradient Discrepancy}
\label{sec:appendix_contraction}

The following proposition gives sufficient local conditions under which Assumption~1 can hold.
We then provide qualitative empirical evidence that a decreasing discrepancy trend can appear during training.

\begin{proposition}[Sufficient conditions for eventual contraction of gradient discrepancy]
\label{prop:sufficient-grad-discrepancy-contraction}
Let $F(\theta)$ and $G(\theta)$ be differentiable objective functions with a common stationary point $\theta^\ast$
such that
\[
\nabla F(\theta^\ast)=\nabla G(\theta^\ast)=0.
\]
Define the discrepancy function $\Delta(\theta):=F(\theta)-G(\theta)$ and note that $\nabla\Delta(\theta^\ast)=0$.

Assume there exists a neighborhood $U$ of $\theta^\ast$ and an index $t_0$ such that the iterates
$(\theta_t)_{t\ge 0}$ produced by the optimization algorithm satisfy $\theta_t\in U$ for all $t\ge t_0$ and:

\begin{enumerate}
\item[\textbf{S1.}] \textbf{(Eventual linear convergence of iterates)} There exists $q\in(0,1)$ such that for all
$t\ge t_0$,
\[
\|\theta_{t+1}-\theta^\ast\| \le q\,\|\theta_t-\theta^\ast\|.
\]
\item[\textbf{S2.}] \textbf{(Local sandwich bound for $\nabla\Delta$)} There exist constants $L_\Delta>0$ and
$\mu_\Delta>0$ such that for all $\theta\in U$,
\[
\mu_\Delta \|\theta-\theta^\ast\| \le \|\nabla\Delta(\theta)\| \le L_\Delta \|\theta-\theta^\ast\|.
\]
\end{enumerate}

Let
\[
\rho := \frac{L_\Delta q}{\mu_\Delta}.
\]
If $\rho<1$, then for all $t\ge t_0$ the gradient discrepancy contracts linearly:
\[
\|\nabla\Delta(\theta_{t+1})\| \le \rho\,\|\nabla\Delta(\theta_t)\|.
\]
In particular, if $\nabla\Delta(\theta_t)\neq 0$, then $\|\nabla\Delta(\theta_{t+1})\|<\|\nabla\Delta(\theta_t)\|$.
\end{proposition}

\begin{proof}
Fix any $t\ge t_0$. Since $\theta_{t+1}\in U$ and $\nabla\Delta(\theta^\ast)=0$, applying the upper bound in
\textbf{S2} (with $\theta=\theta_{t+1}$) yields
\begin{equation}
\label{eq:suff_step1}
\|\nabla\Delta(\theta_{t+1})\|
= \|\nabla\Delta(\theta_{t+1})-\nabla\Delta(\theta^\ast)\|
\le L_\Delta \|\theta_{t+1}-\theta^\ast\|.
\end{equation}
By the eventual linear convergence assumption \textbf{S1},
\begin{equation}
\label{eq:suff_step2}
L_\Delta \|\theta_{t+1}-\theta^\ast\|
\le L_\Delta q \|\theta_t-\theta^\ast\|.
\end{equation}
Finally, since $\theta_t\in U$, applying the lower bound in \textbf{S2} (with $\theta=\theta_t$) and rearranging gives
\begin{equation}
\label{eq:suff_step3}
\mu_\Delta \|\theta_t-\theta^\ast\| \le \|\nabla\Delta(\theta_t)\|
\quad\Longrightarrow\quad
\|\theta_t-\theta^\ast\| \le \frac{1}{\mu_\Delta}\|\nabla\Delta(\theta_t)\|.
\end{equation}
Combining \eqref{eq:suff_step1}, \eqref{eq:suff_step2}, and \eqref{eq:suff_step3} yields
\[
\|\nabla\Delta(\theta_{t+1})\|
\le L_\Delta \|\theta_{t+1}-\theta^\ast\|
\le L_\Delta q \|\theta_t-\theta^\ast\|
\le \frac{L_\Delta q}{\mu_\Delta}\|\nabla\Delta(\theta_t)\|
= \rho\,\|\nabla\Delta(\theta_t)\|,
\]
which proves the stated contraction. If $\rho<1$ and $\nabla\Delta(\theta_t)\neq 0$, then the inequality is strict.
\end{proof}

Proposition~\ref{prop:sufficient-grad-discrepancy-contraction} is stated as a set of sufficient local conditions
under which the gradient discrepancy contracts eventually (i.e., for all $t\ge t_0$ once the iterates remain in
a neighborhood $U$ of a stationary point $\theta^\ast$): \textbf{S1} only requires eventual linear convergence of the optimization iterates in parameter space, and
\textbf{S2} requires a local ``sandwich'' bound on $\|\nabla\Delta(\theta)\|$ around $\theta^\ast$—an upper bound that
is consistent with local smoothness and a lower bound that corresponds to a local error-bound / gradient-growth
condition. While such conditions are classical in optimization theory and may not hold globally for modern
overparameterized neural networks, they can be locally valid once training enters a stable basin around a well-behaved
solution. In particular, local error-bound behavior (which rules out extreme flatness in the discrepancy landscape)
is often observed empirically near converged solutions, and it is strictly weaker than requiring global strong
convexity.

\paragraph{Empirical evidence.}
To qualitatively assess whether an eventual contraction pattern is observed in practice, we trained a two-layer
network (learning rate $10^{-4}$) on three datasets. For each dataset, we formed $S$ by uniformly sampling 1000 points
and formed $S_J$ as a uniformly random subset of $S$. We then monitored the discrepancy sequence
$\|\mathcal{DF}(S,S_J,t)\|=\|\nabla\Delta(\theta_t)\|$ over epochs $t=1,\ldots,150$.
\begin{figure}[H]
\centering
\begin{minipage}{0.32\textwidth}
    \centering
    \includegraphics[width=\textwidth]{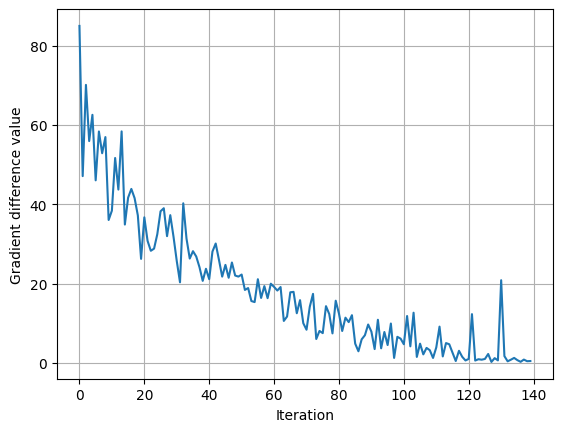}
    \captionsetup{font=small}
    \subcaption{German}
    \label{fig:german-diff}
\end{minipage}
\hfill
\begin{minipage}{0.32\textwidth}
    \centering
    \includegraphics[width=\textwidth]{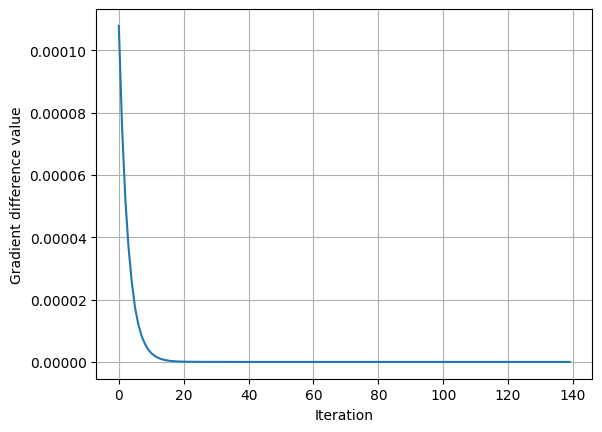}
    \captionsetup{font=small}
    \subcaption{Gisette}
    \label{fig:gisset-diff}
\end{minipage}
\hfill
\begin{minipage}{0.32\textwidth}
    \centering
    \includegraphics[width=\textwidth]{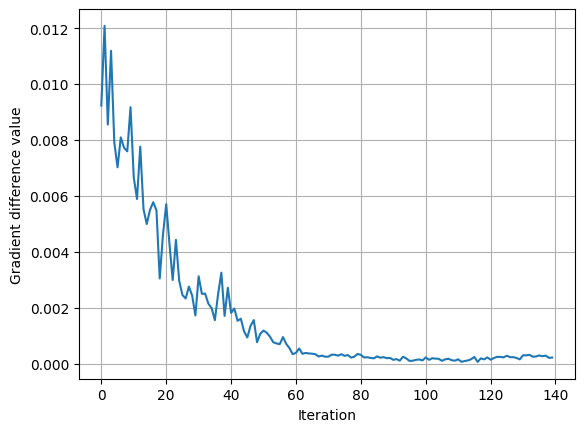}
    \captionsetup{font=small}
    \subcaption{20~Newsgroups}
    \label{fig:newsgroup-diff}
\end{minipage}

\caption{$\mathcal{DF}$ values over training epochs for German, Gisette, and 20~Newsgroups.}
\label{fig:bigdatasetd}
\end{figure}
Across Gisette and 20~Newsgroups the discrepancy decreases over training and typically
stabilizes (approaching a near-constant plateau) in later epochs. This behavior is consistent with the conclusion of
Proposition~\ref{prop:sufficient-grad-discrepancy-contraction}: once the iterates enter a stable neighborhood $U$ where
\textbf{S1}--\textbf{S2} are approximately satisfied, the discrepancy should contract with an effective rate
$\rho=\frac{L_\Delta q}{\mu_\Delta}<1$. For German, the overall trend is decreasing but with occasional epochs that
violate strict monotonicity, suggesting that the ``eventual'' regime may be intermittently left or that the effective
local constants vary during training.
\section{Additional Learning Curves}
\label{app:al_curves}
The remaining learning curves are presented in Figure~\ref{fig:lcurve-appendix}.

\begin{figure}[h]
    \centering

    \begin{subfigure}[b]{0.48\textwidth}
        \centering
        \includegraphics[width=\textwidth]{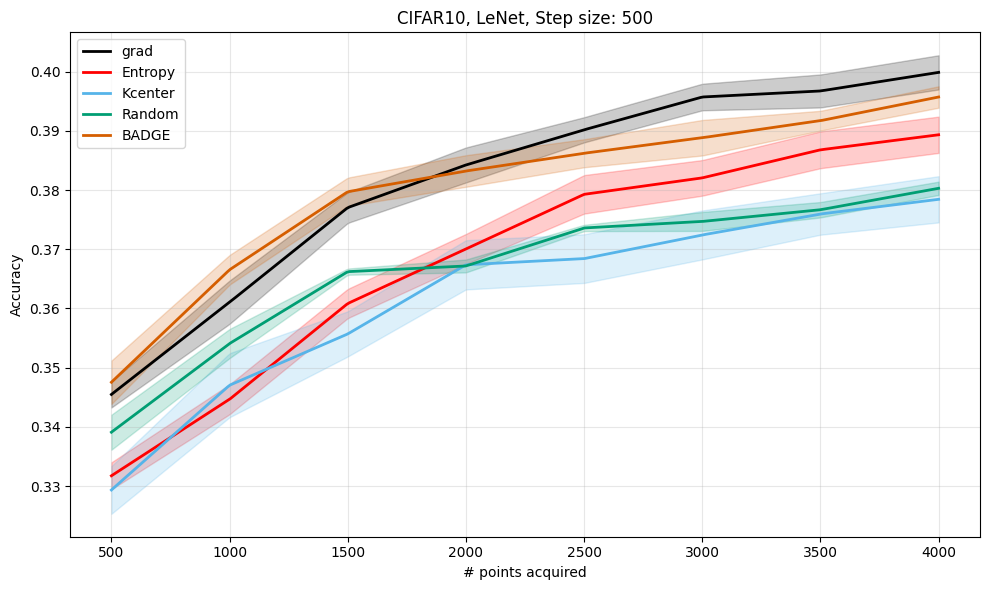}
        \caption{}
    \end{subfigure}\hfill
    \begin{subfigure}[b]{0.48\textwidth}
        \centering
        \includegraphics[width=\textwidth]{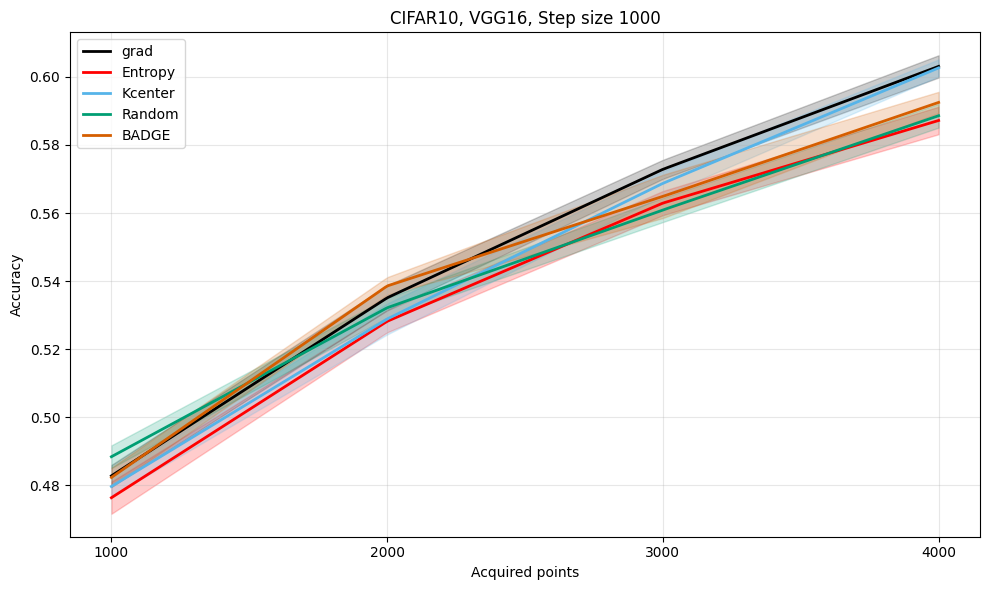}
        \caption{}
    \end{subfigure}

    \vspace{0.5cm}

    \begin{subfigure}[b]{0.48\textwidth}
        \centering
        \includegraphics[width=\textwidth]{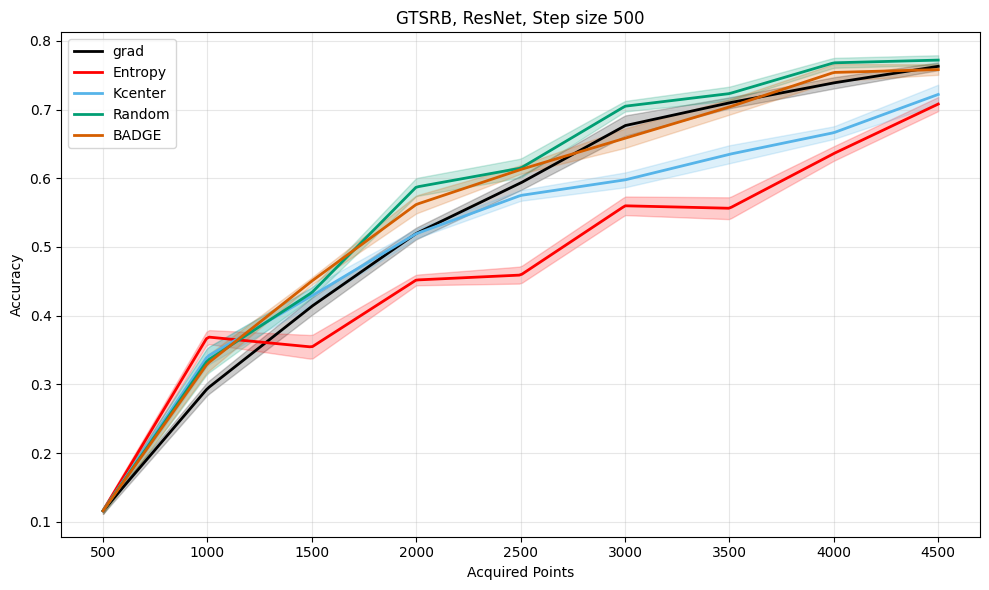}
        \caption{}
    \end{subfigure}\hfill
    \begin{subfigure}[b]{0.48\textwidth}
        \centering
        \includegraphics[width=\textwidth]{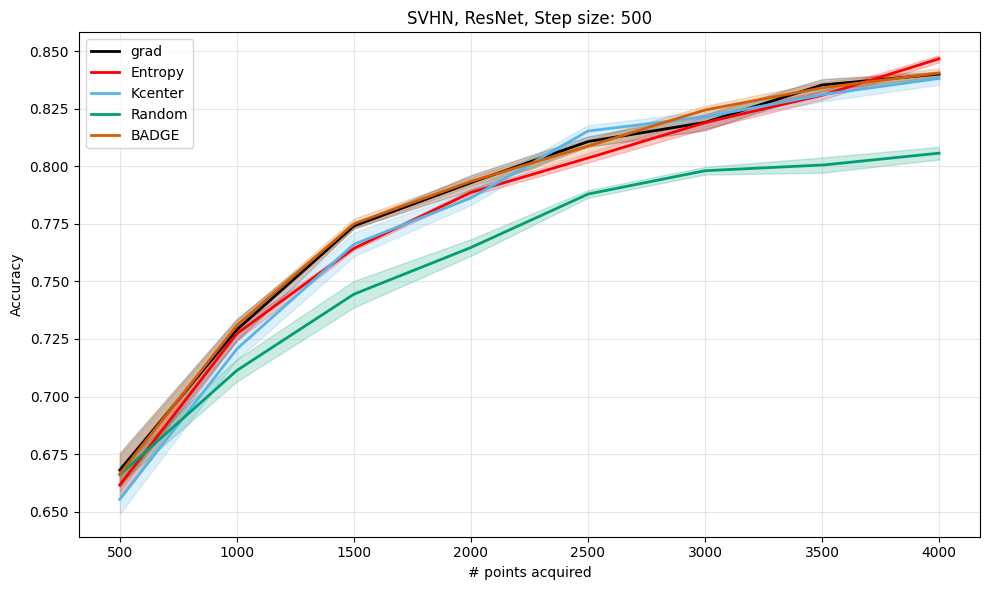}
        \caption{}
    \end{subfigure}

    \caption{Active learning test accuracy on image benchmarks, mean over five seeds; shaded regions indicate variability.}
    \label{fig:lcurve-appendix}
\end{figure}

\begin{figure}[h]
    \ContinuedFloat
    \centering

    \begin{subfigure}[b]{0.48\textwidth}
        \centering
        \includegraphics[width=\textwidth]{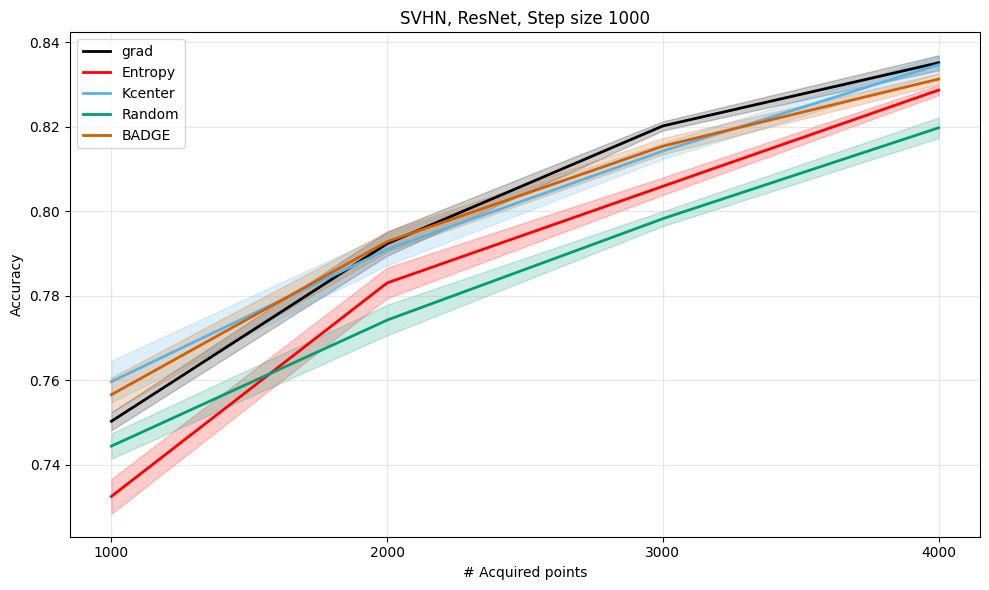}
        \caption{}
    \end{subfigure}\hfill
    \begin{subfigure}[b]{0.48\textwidth}
        \centering
        \includegraphics[width=\textwidth]{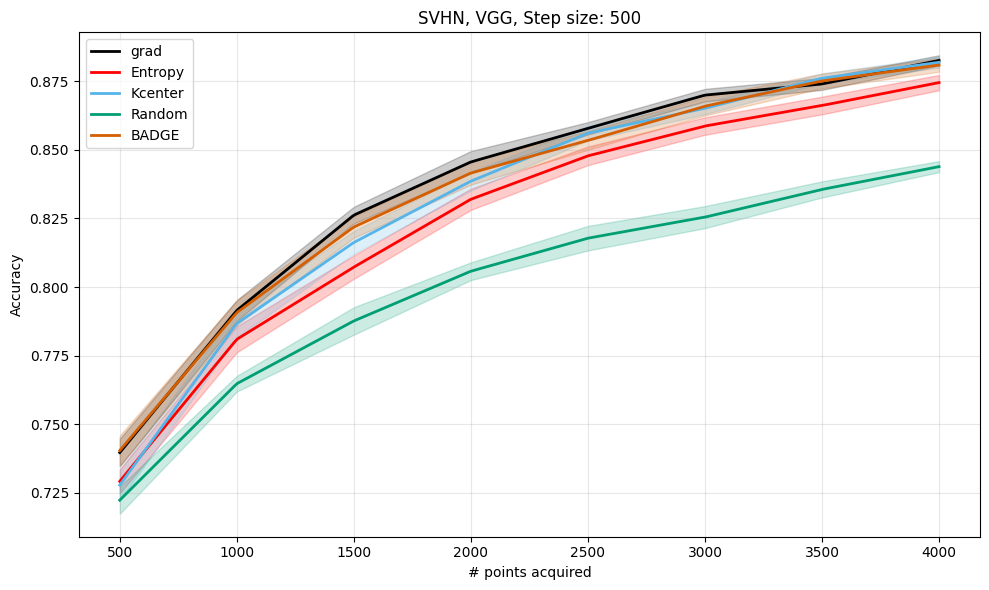}
        \caption{}
    \end{subfigure}

    \vspace{0.5cm}

    \begin{subfigure}[b]{0.48\textwidth}
        \centering
        \includegraphics[width=\textwidth]{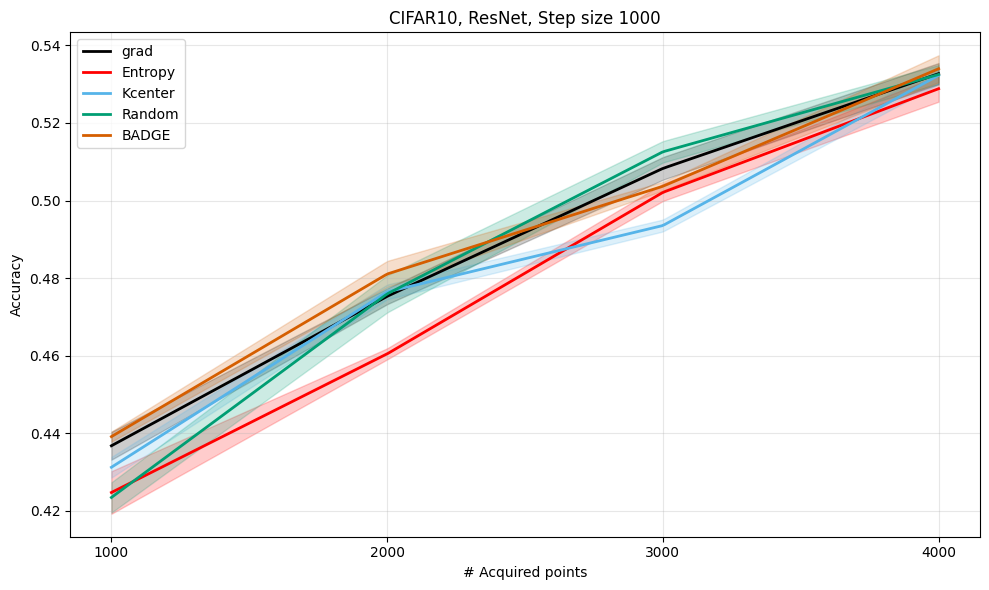}
        \caption{}
    \end{subfigure}\hfill
    \begin{subfigure}[b]{0.48\textwidth}
        \centering
        \includegraphics[width=\textwidth]{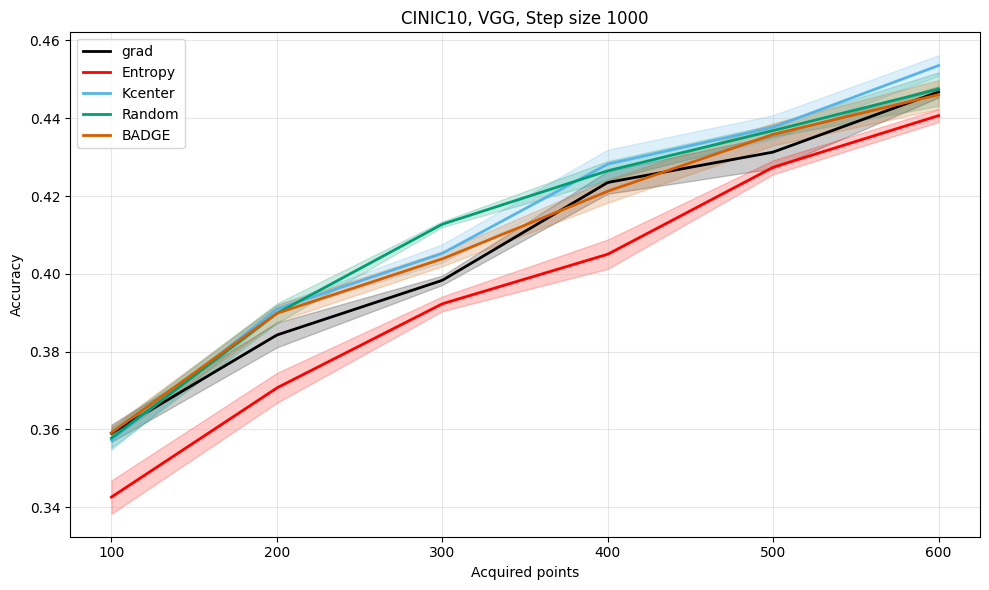}
        \caption{}
    \end{subfigure}

    \caption[]{Active learning test accuracy on image benchmarks, mean over five seeds; shaded regions indicate variability, continued.}
\end{figure}

\begin{figure}[p]
    \ContinuedFloat
    \centering

    \begin{subfigure}[b]{0.6\textwidth}
        \centering
        \includegraphics[width=\textwidth]{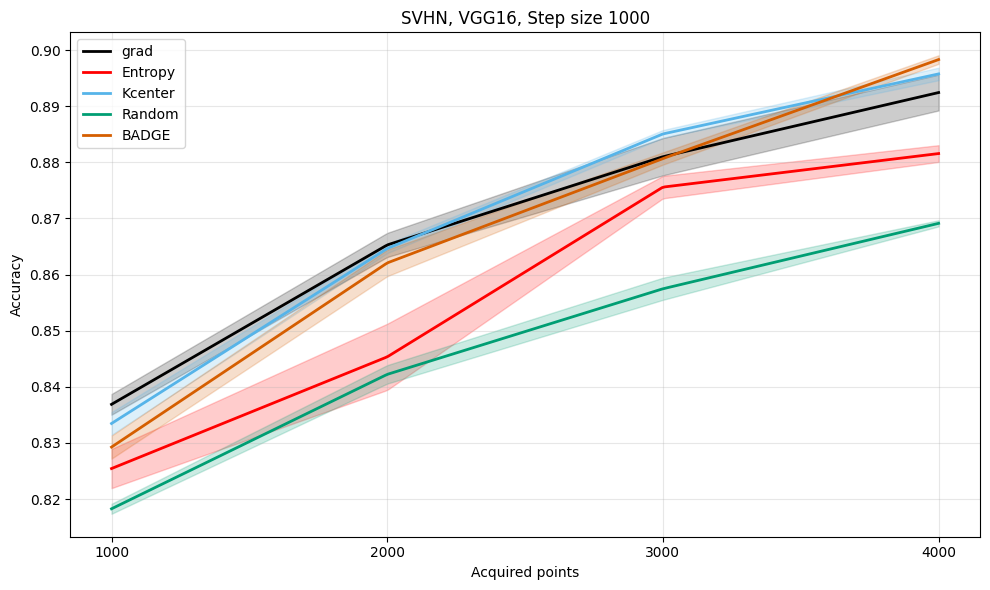}
        \caption{}
    \end{subfigure}

    \caption[]{Active learning test accuracy on image benchmarks, mean over five seeds; shaded regions indicate variability, continued.}
\end{figure}
\section{Pairwise Comparisons}
\label{App:comp}

In Section~\ref{sec:overall}, we compared methods over all experiments and rounds, as well as separately over
the first three and last three rounds of each experiment. Here, we further separate the pairwise penalty matrix
and loss-score plots by model family in Figure~\ref{fig:model} and by dataset in Figure~\ref{fig:ppm_dataset}. 
\begin{figure}[htbp]
    \centering
    \begin{minipage}[b]{0.32\textwidth}
        \centering
        \includegraphics[width=\textwidth]{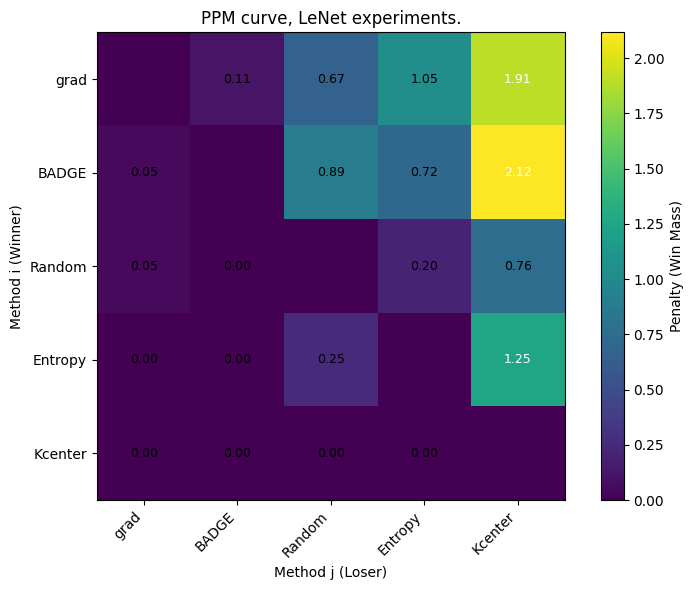}
        \caption*{(a)}
    \end{minipage}\hfill
    \begin{minipage}[b]{0.32\textwidth}
        \centering
        \includegraphics[width=\textwidth]{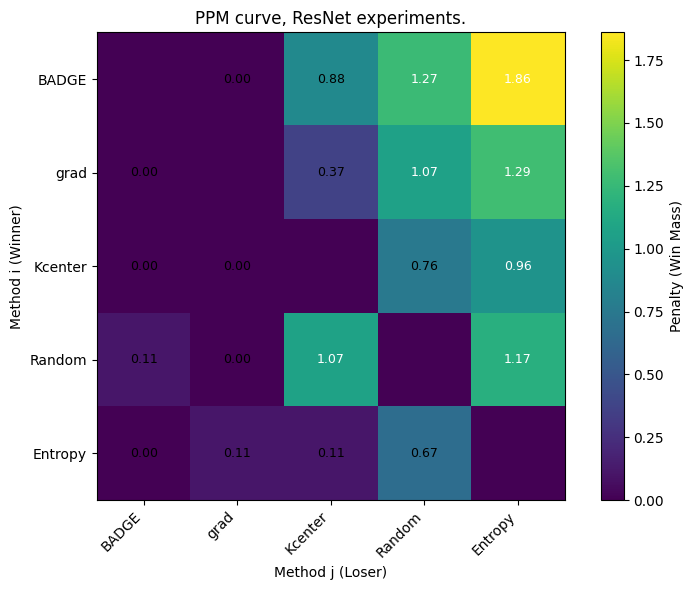}
        \caption*{(b)}
    \end{minipage}\hfill
    \begin{minipage}[b]{0.32\textwidth}
        \centering
        \includegraphics[width=\textwidth]{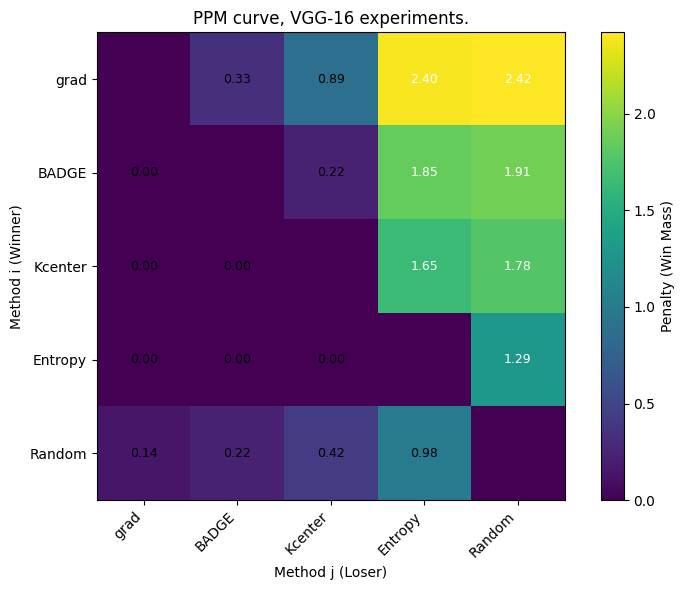}
        \caption*{(c)}
    \end{minipage}

    \vspace{0.4cm}

    \begin{minipage}[b]{0.32\textwidth}
        \centering
        \includegraphics[width=\textwidth]{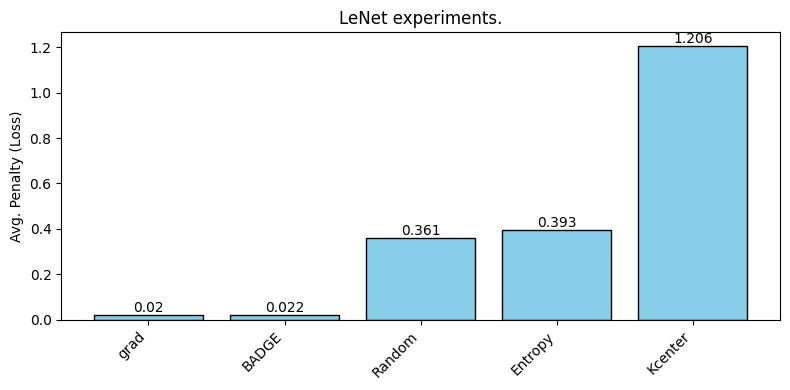}
        \caption*{(d)}
    \end{minipage}\hfill
    \begin{minipage}[b]{0.32\textwidth}
        \centering
        \includegraphics[width=\textwidth]{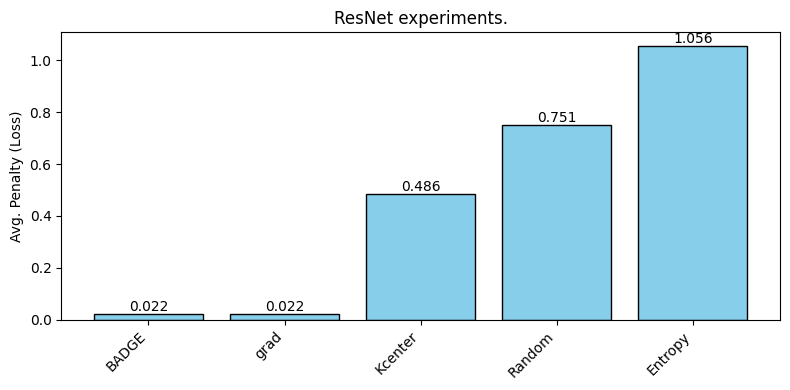}
        \caption*{(e)}
    \end{minipage}\hfill
    \begin{minipage}[b]{0.32\textwidth}
        \centering
        \includegraphics[width=\textwidth]{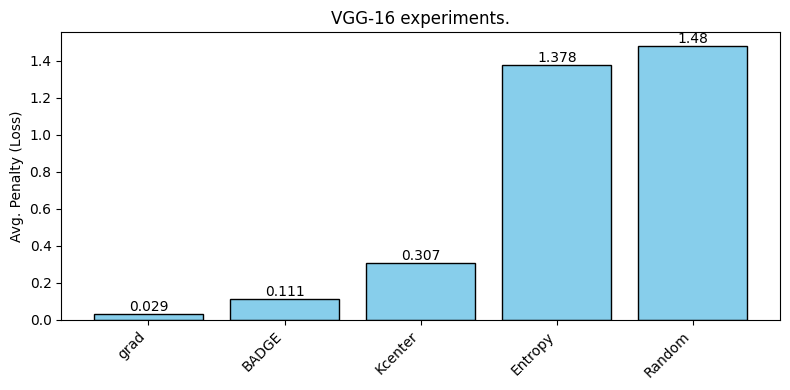}
        \caption*{(f)}
    \end{minipage}

    \caption{Overall comparison using the pairwise penalty matrix (top row) and the corresponding loss-score ranking (bottom row). (a,d) LeNet experiments; (b,e) ResNet experiments; (c,f) VGG-16 experiments. Larger PPM entries indicate more frequent statistically significant wins, while lower loss scores indicate stronger overall performance.}
    \label{fig:model}
\end{figure}
\begin{figure}[htbp]
    \centering
    \begin{minipage}[b]{0.32\textwidth}
        \centering
        \includegraphics[width=\textwidth]{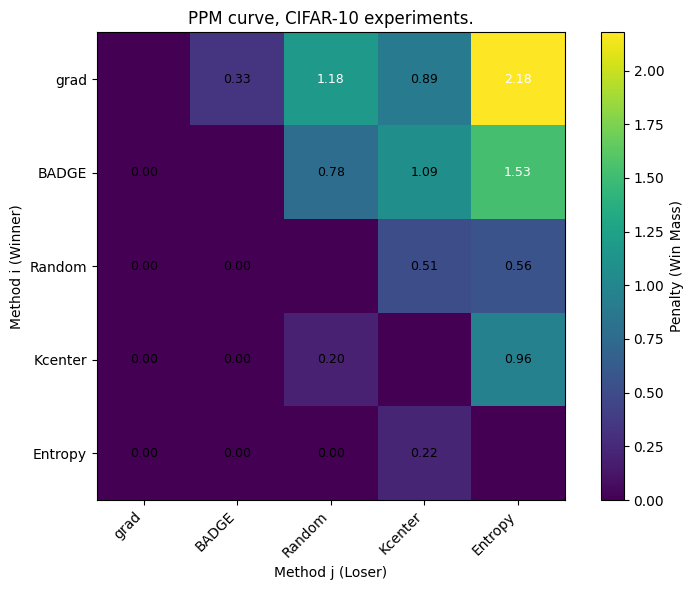}
        \caption*{(a)}
    \end{minipage}\hfill
    \begin{minipage}[b]{0.32\textwidth}
        \centering
        \includegraphics[width=\textwidth]{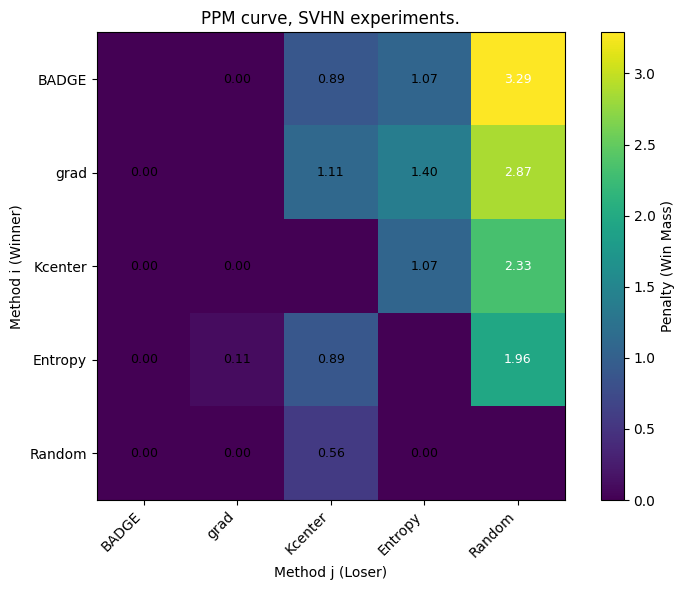}
        \caption*{(b)}
    \end{minipage}\hfill
    \begin{minipage}[b]{0.32\textwidth}
        \centering
        \includegraphics[width=\textwidth]{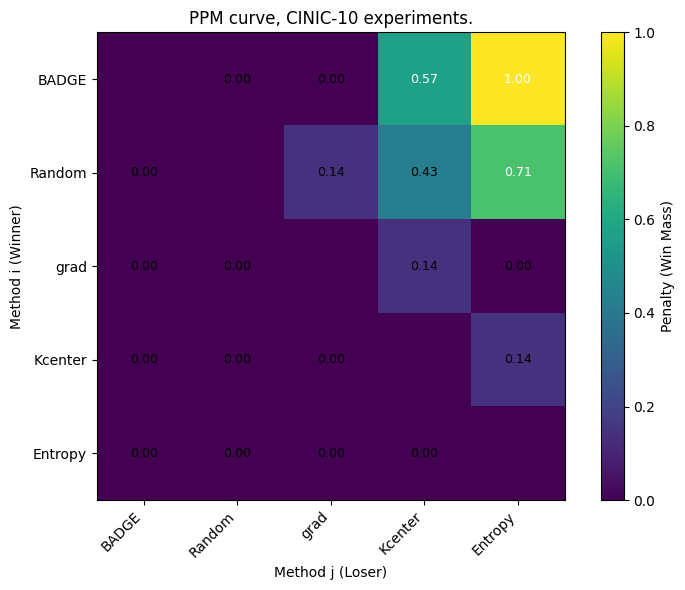}
        \caption*{(c)}
    \end{minipage}

    \vspace{0.4cm}

    \begin{minipage}[b]{0.32\textwidth}
        \centering
        \includegraphics[width=\textwidth]{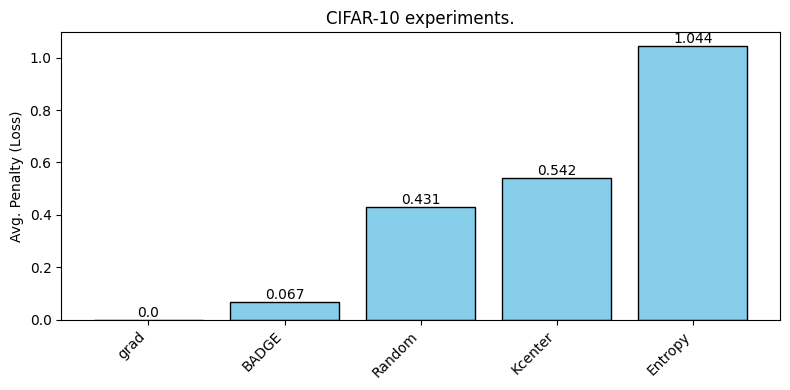}
        \caption*{(d)}
    \end{minipage}\hfill
    \begin{minipage}[b]{0.32\textwidth}
        \centering
        \includegraphics[width=\textwidth]{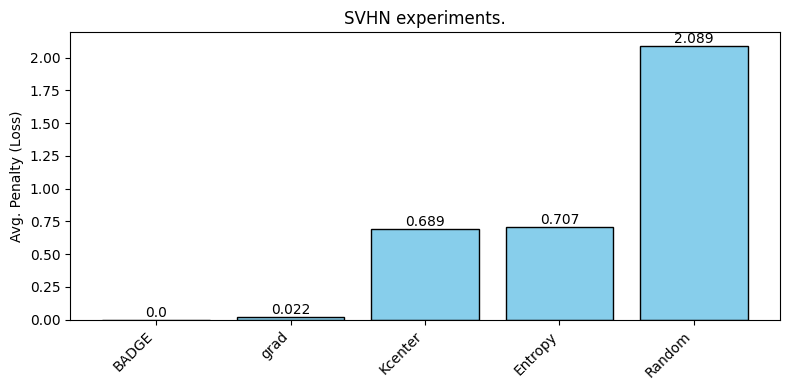}
        \caption*{(e)}
    \end{minipage}\hfill
    \begin{minipage}[b]{0.32\textwidth}
        \centering
        \includegraphics[width=\textwidth]{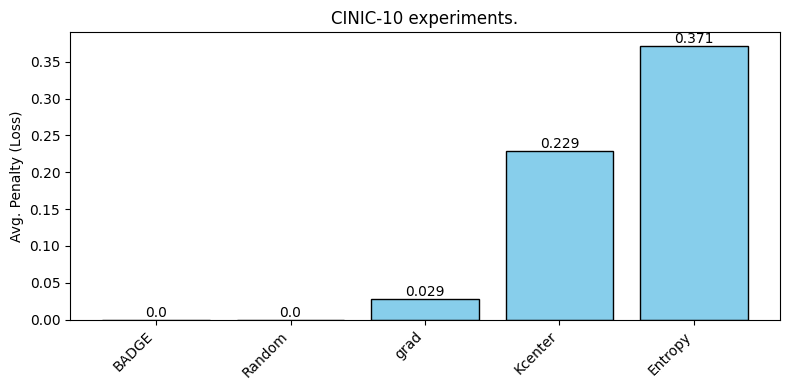}
        \caption*{(f)}
    \end{minipage}

    \caption{Overall comparison using the pairwise penalty matrix (top row) and the corresponding loss-score ranking (bottom row). (a,d) CIFAR-10 experiments; (b,e) SVHN experiments; (c,f) CINIC-10 experiments. Larger PPM entries indicate more frequent statistically significant wins, while lower loss scores indicate stronger overall performance.}
    \label{fig:ppm_dataset}
\end{figure}

\end{document}